\documentclass[journal]{IEEEtran}

\usepackage{amssymb}
\usepackage{amsmath}
\usepackage{graphicx}
\usepackage[abs]{overpic}
\usepackage{cite}
\usepackage[table]{xcolor}
\usepackage{multirow}
\usepackage{color}
\usepackage{amsthm}
\usepackage{MnSymbol}
\usepackage{tabularx}
\usepackage{dsfont}
\usepackage{pifont}
\usepackage{tikz}
\usepackage{algorithm}
\usepackage{algorithmic}
\usetikzlibrary{patterns}
\usetikzlibrary{decorations.pathmorphing}
\usepackage{stfloats}
\usepackage{arydshln}
\usepackage{multirow}
\usepackage{amsthm}
\usepackage{subcaption}
\usepackage{epstopdf}
\usepackage{pifont}
\usepackage[utf8]{inputenc}

\definecolor{myblue}{rgb}{.106,.2,.322}
\definecolor{myorange}{rgb}{.984,.392,.016}

\newtheorem{lemma}{Lemma}
\newtheorem{proposition}{Proposition}
\newtheorem{corollary}{Corollary}

\theoremstyle{remark}
\newtheorem*{remark}{Remark}

\newcounter{MYtempeqncnt}

\newcommand{\nn}{\nonumber}

\newcommand{\Z}{\mathbb{Z}}

\newcommand{\player}[1]{$\mathcal{P}_{#1}$}

\newcommand{\indicator}{\mathds{1}}

\newcommand{\R}{\mathbb{R}}

\newcommand{\symbolsSet}{\Sigma^n}
\newcommand{\trueSymbols}{\symbolsSet_{o}}
\newcommand{\decodableSymbols}{\symbolsSet_{d}}
\newcommand{\noisySet}{\Sigma^n}
\newcommand{\trueSymbol}{x_o}
\newcommand{\justSymbol}{x}
\newcommand{\noisySymbol}{y}
\newcommand{\decoderOutput}{\hat{\noisySymbol}}

\newcommand{\factor}[1]{\gamma_{#1}}

\newcommand{\Aactiono}{a_{o}}
\newcommand{\Aactiona}{a_{\justSymbol}}
\newcommand{\Aaction}{a}
\newcommand{\Amixed}{\alpha}
\newcommand{\AactionSpace}{{\cal A}}
\newcommand{\AmixedSpace}{\Delta^{|\AactionSpace|-1}}

\newcommand{\Dmixed}{\pi}

\newcommand{\DmixedSpace}{[0,1]^{d_o+1}}
\newcommand{\channel}{p(\noisySymbol|\justSymbol)}

\newcommand{\game}{{\cal G}}
\newcommand{\loss}{\ell}

\newcommand{\bestResponse}{B}
\newcommand{\errorProb}{p_e}

\newcommand{\reward}{r}
\newcommand{\Hamming}{H}
\newcommand{\allHamming}{h}

\newcommand{\vOnes}{\mathbf{1}}
\newcommand{\vZeros}{\mathbf{0}}

\newcommand{\gameMatrix}{\Xi}
\newcommand{\gameConstant}{\xi}

\newcommand{\cc}{\cellcolor{myorange!20!white}}

\begin{document}

\title{Reliable Smart Road Signs}

\author{Muhammed O. Sayin, Chung-Wei Lin, Eunsuk Kang,\\\hspace{.2in}Shinichi Shiraishi, and Tamer Ba\c{s}ar, \IEEEmembership{Life Fellow, IEEE}
\thanks{This research was partially supported by the U.S. Office of Naval Research (ONR) MURI grant N00014-16-2710.}
\thanks{M. O. Sayin and T. Ba\c{s}ar are with the Department
of Electrical and Computer Engineering, University of Illinois at Urbana-Champain, Urbana,
IL, 61801, USA (e-mail: \{sayin2,basar1\}@illinois.edu).}
\thanks{C.-W. Lin is with National Taiwan University, Taipei, Taiwan (email: cwlin@csie.ntu.edu.tw).}
\thanks{E. Kang is with Carnegie Mellon University, Pittsburgh, PA, 15213 USA (email: eskang@cmu.edu).}
\thanks{S. Shiraishi is with Toyota InfoTechonology Center Co., Ltd., Minato-ku, Tokyo, 107-0052, Japan (e-mail: sshiraishi@jp.toyota-itc.com).}}


\maketitle

\begin{abstract}
In this paper, we propose a game theoretical adversarial intervention detection mechanism for reliable smart road signs. A future trend in intelligent transportation systems is ``smart road signs" that incorporate smart codes (e.g., visible at infrared) on their surface to provide more detailed information to smart vehicles. Such smart codes make road sign classification problem aligned with communication settings more than conventional classification. This enables us to integrate well-established results in communication theory, e.g., error-correction methods, into road sign classification problem. Recently, vision-based road sign classification algorithms have been shown to be vulnerable against (even) small scale adversarial interventions that are imperceptible for humans. On the other hand, smart codes constructed via error-correction methods can lead to robustness against small scale intelligent or random perturbations on them. In the recognition of smart road signs, however, humans are out of the loop since they cannot see or interpret them. Therefore, there is no equivalent concept of imperceptible perturbations in order to achieve a comparable performance with humans. Robustness against small scale perturbations would not be sufficient since the attacker can attack more aggressively without such a constraint. Under a game theoretical solution concept, we seek to ensure certain measure of guarantees against even the worst case (intelligent) attackers that can perturb the signal even at large scale. We provide a randomized detection strategy based on the distance between the decoder output and the received input, i.e., error rate. Finally, we examine the performance of the proposed scheme over various scenarios.
\end{abstract}


\IEEEpeerreviewmaketitle

\section{Introduction}

\IEEEPARstart{M}{achine} learning is one of the  key enabling technologies for autonomous vehicles. An autonomous vehicle can learn how to recognize the surroundings and can base its strategic decisions on the information learnt. It is only a matter of time for autonomous driving to replace of human drivers completely. However, for the time being, there are still important, yet not completely addressed, challenges for autonomous driving. Road-sign classification is one of these challenges. Varying weather conditions, changing lighting throughout the day and occlusion are known to pose challenges to road-sign recognition/classification in real-time \cite{ref:Mogelmose12}. However, recently, it has been shown that there can also be physical adversarial modifications, e.g., stickers or graffiti, on the road signs to mislead the classification algorithms \cite{ref:Eykholt18}. 

\subsection{Prior Literature}

In the field of intelligent transportation systems, there have been extensive effort to mitigate the former challenge \cite{ref:Jin14, ref:Gonzalez14, ref:Greenhalgh15, ref:Yang16, ref:Zeng17, ref:Liu14, ref:Lu17a}. In \cite{ref:Jin14}, the authors have studied convolutional neural networks trained according to hinge loss stochastic gradient descent to achieve fast and stable convergence rates with substantial recognition performance. In \cite{ref:Gonzalez14} and \cite{ref:Greenhalgh15}, the authors have proposed text-based detection systems for traffic panels that could include information that can vary substantially. Computational complexity of the recognition algorithms plays a significant role for real-time applications since autonomous vehicles are time-critical systems \cite{ref:Yang16,ref:Zeng17}. In \cite{ref:Yang16}, the authors have sought to enhance the performance of convolutional neural networks for faster performance in real-time applications through localization of the traffic-signs in the input images based on their types. In \cite{ref:Zeng17}, the authors have proposed kernel-based extreme learning machines with deep perceptual features to achieve comparable performance to hinge-loss stochastic gradient based convolutional neural networks (proposed in \cite{ref:Jin14}) with reduced computational complexity. Tree-based hierarchical structures have also been proposed to achieve coarse-to-fine road sign detection \cite{ref:Liu14,ref:Lu17a}.

Different from the previous works \cite{ref:Jin14, ref:Gonzalez14, ref:Liu14, ref:Greenhalgh15, ref:Yang16, ref:Lu17a, ref:Zeng17}, however, in this paper, we seek to address the latter challenge, i.e., road-sign classification in {\em adversarial} environments, where there can be an intelligent attacker modifying road signs physically as exemplified in  \cite{ref:Eykholt18}. Especially for vision-based classification algorithms, it is an important issue that an attacker could craft the input through perturbations that are {\em imperceptible} for humans (i.e., a human would still easily classify the input correctly) while the algorithm classifies the input as the attacker's targeted class \cite{ref:Szegedy14, ref:Goodfellow15, ref:Athalye18a, ref:Athalye18b, ref:Eykholt18}. Such an input sample is called {\em adversarial example} \cite{ref:Szegedy14}. Recently, substantial amount of defense methods have been proposed to make machine learning algorithms robust against adversarial scenarios. These defense methods have been developed to provide robustness against certain classes of attacks, and it has been shown that it is possible to bypass them all via small modification of the attacks \cite{ref:Athalye18a}. 

\begin{figure*}[t!]
\centering
\begin{subfigure}[c]{.165\textwidth}
\includegraphics[width =\textwidth]{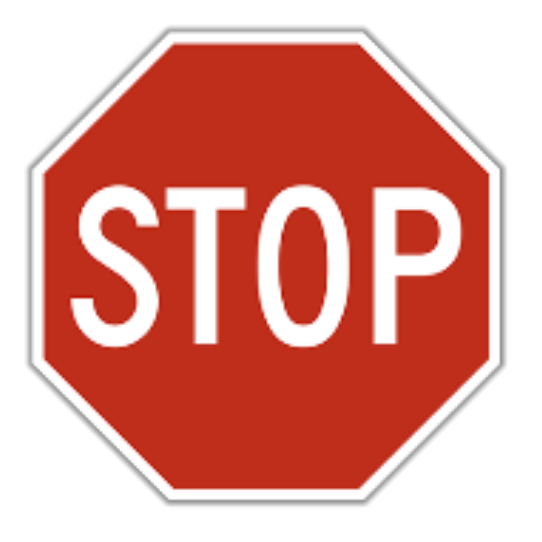}
\caption{Visual Image}
\end{subfigure}~~~~
\begin{subfigure}[c]{.175\textwidth}
\includegraphics[width =\textwidth]{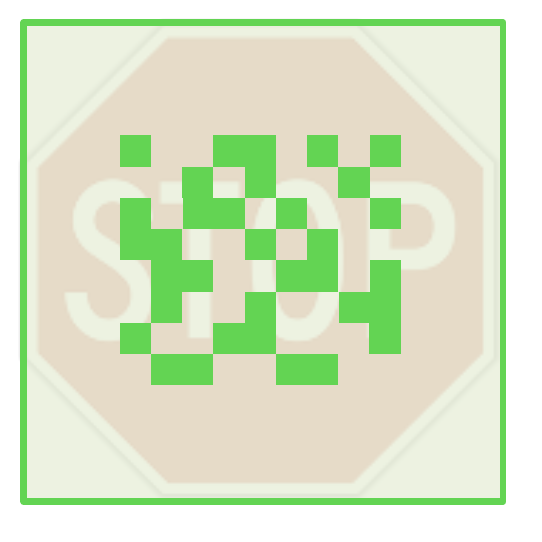}
\caption{Infrared Image of Smart Code}\label{fig:infrared}
\end{subfigure}
\begin{subfigure}[c]{.2\textwidth}
\includegraphics[width = \textwidth]{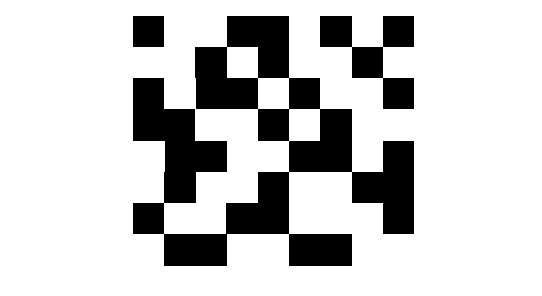}
\caption{Original Codeword}\label{fig:code}
\end{subfigure}
\begin{subfigure}[c]{.175\textwidth}
\includegraphics[width = \textwidth]{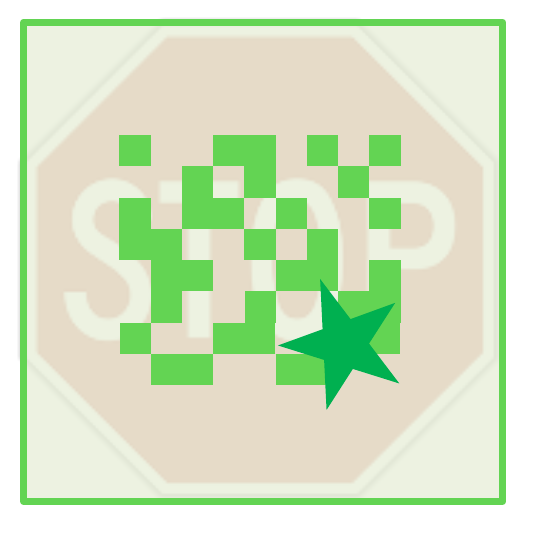}
\caption{Infrared Image of Smart Code Attacked}\label{fig:attack}
\end{subfigure}
\begin{subfigure}[c]{.2\textwidth}
\includegraphics[width = \textwidth]{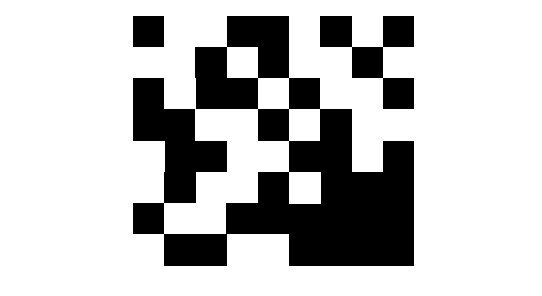}
\caption{Codeword Attacked}\label{fig:attackedCode}
\end{subfigure}
\caption{Examples of a road sign with smart code and an attack (an infrared star-shaped drawing on the right-bottom corner).} \label{fig:stop}
\end{figure*}

\subsection{Smart Road Signs}

Our goal, here, is to achieve reliable identification of road signs by smart vehicles without limiting the solution to learning-based techniques. Particularly, we can view the road-sign classification problem from a wider perspective as a {\em communication} problem. The road sign and smart vehicle can be viewed as a transmitter and a receiver, respectively. Then, the message is the type of the road sign, the signal carrying that message is the physical road sign, and the signal received is its digital image taken by the smart vehicle. The relation in between the signals sent and received, i.e., physical road sign and its digital image, can be modeled via a noisy channel that can lead to errors in the transmitted message. Based on this viewpoint, we can reconfigure this information flow by also designing the signal, i.e., the physical road sign, via the well-established tools developed in communication theory. 

How we can re-design physical road signs and which medium we can choose to transmit information are only limited by our imagination and the corresponding financial burden to adapt the infrastructure. For example, at each road sign, we could have included road side units that can transmit the message via dedicated short range communication (DSRC) radios \cite{ref:IEEE1609-2,ref:Kenney11} although there might also be adversarial interventions on those radio signals \cite{ref:Kenney11,ref:Liang18}. However, the future trend in road signs is to include (e.g., infrared) smart codes that can be read by smart vehicles as illustrated figuratively in Fig. \ref{fig:infrared} \cite{ref:Schwab17,ref:Hyatt18}. Such smart codes provide flexibility in transmiting more information to smart vehicles instead of just the type of road signs. 

\begin{remark}
Due to the regularity of smart codes, e.g., see Fig. \ref{fig:stop}, we can identify them via image processing techniques instead of learning-based techniques. For example, in \cite{ref:Fan12}, the authors use a simularity measure to identify the countdown numbers in a traffic light, instead of traditional number recognition algorithms. Furthermore, quick response (QR) codes constitute another example where we can incorporate image processing techniques to identify the underlying code, e.g., \cite{ref:QR}. 
\end{remark}

We can attain reliable information transmission with certain formal guarantees when we construct the smart codes via error-correction methods. Particularly, in coding theory, error correction methods introduce redundancy to signals, i.e., codewords, in order to recover the underlying message as accurately as possible when there is a noisy channel that can perturb the signal \cite{ref:Blahut02}. Correspondingly, if the amount of perturbation on the smart code, due to some physical modifications as seen in Fig. \ref{fig:attack}, were less than half of the minimum distance between any two codewords, then we could have recovered the underlying codeword without any error. However, this is not the end of the story as explained below.

\subsection{Motivation}

Recall that physical adversarial examples in visual tasks are defined as inputs crafted by an intelligent attacker in order to mislead the classification algorithms while a human can still classify it accurately without any difficulty \cite{ref:Eykholt18,ref:Athalye18b}. This challenge is important to mitigate in order to attain a comparable performance with human drivers. However, such a threat model is not appropriate for smart codes since they will not be visible to humans and even if they were visible, they are not interpretable by humans manually easily. Since humans are out of the loop, there will not be such a constraint limiting the perturbation amount on the attack that we are seeking to defend against. In the literature of communication in adversarial environments \cite{ref:Li15,ref:Zhang16}, it is evident that it will not be sufficient for smart codes to be robust against just small scale perturbations if the adversary is powerful to launch large scale perturbation. Therefore, we introduce a new threat model where the intelligent adversary can also perturb the input at large scale in order to lead to erroneous decoding.  

Given a codeword received, our goal is to detect whether there exists an adversarial intervention or not, e.g., as illustrated in Fig. \ref{fig:attackedCode}. Note that the perceived codeword might differ from the original codeword not only due to adversarial intervention but also due to random perturbations that inevitably appear in the process of interpreting the infrared image of the smart code. Indeed, the presence of such random perturbations is the main reason to incorporate error correction methods while constructing the smart codes. Note also that an intelligent attacker attacks by taking the detection mechanisms into account. Correspondingly, while designing the detection strategy, we need to anticipate the reaction of the attacker. However, one-level depth reasoning where the detector designs the strategy by anticipating the reaction of the attacker would not be effective if the attacker has taken this proactive defense into account and has reacted in a way that can undermine it. In \cite{ref:Athalye18a}, the authors have shown this phenomenon by bypassing the state-of-the-art defense mechanisms through strategic modification of the attack (those mechanisms defend against). To mitigate this issue, we propose to design the detection strategies under the solution concept of game theoretical equilibrium \cite{ref:Basar99}. However, this paper is definitely not the first one approaching the adversarial classification (or intervention detection) problem through a game-theoretical lens. In the following, we review these studies.  

\subsection{Game Theoretical Approaches}

In \cite{ref:Dalvi04}, the authors have introduced a non-zero sum game between an attacker and a classifier; however, they have not studied any notion of equilibrium. In \cite{ref:Bruckner09,ref:Bruckner11}, the authors have studied adversarial prediction problems for a certain class of learners, e.g., support-vector-machines, in terms of Nash and Stackelberg equilibria, respectively. Recently, in \cite{ref:Dritsoula17}, the authors have analyzed adversarial binary-classification as a non-zero sum game between an attacker and a classifier. The classifier seeks to detect whether the input is coming from the attacker or from a known benign-distribution. On the other side, the attacker seeks to maximize his reward (which depends on the input) without being detected by the classifier. The authors have shown that the classifier can restrict the strategies to mixtures of classifiers setting threshold on the reward of the attacker without any loss of generality. However, the results in \cite{ref:Dritsoula17} cannot be extended to our problem setting since the codeword attacked can also be perturbed randomly in the process of reading it. Due to that randomness, different attacks with different rewards can lead to the same codeword received. Therefore, given the received codeword, the defender cannot know the attacker's reward to compare it against such a threshold.

\subsection{Our Contributions}

In this paper, we propose an adversarial intervention detection mechanism for smart road signs in order to ensure reliable recognition by smart vehicles. We model the interaction between the detection mechanism and attackers as a zero-sum Stackelberg game \cite{ref:Basar99} where the detector is the leader. Particularly, attackers can attack road signs by physically modifying their smart codes at large or small scales, as illustrated in Fig. \ref{fig:attack}, while knowing the detection mechanism. The detector seeks to minimize a performance metric that includes cost of losing the opportunity of preventing future attacks by not being able to detect it now, cost of adversary-induced decoding errors or failures, false alarm cost, and easiness of deceptive perturbations. Against the worst-case attacker who seeks to maximize the detector's performance metric, the detector designs a randomized detection rule based on the distance between the received codeword and the decoder output, i.e., error rate. 

The game theoretical solution concept yields that the detector needs to anticipate the attacker's best reaction to the proposed detection policy. However, large size of the attacker's strategy space can lead to computational issue while computing the best detection rule offline. To this end, we examine the attacker's actions and show that the attacker can be viewed as selecting an action from a quotient space of the actual attack space with respect to a certain equivalence relation, which will be described in detail in Subsection \ref{sec:equivalence}. However, that quotient space can still be large to search over if there are many distinct road signs. Accordingly, we provide a method to relax the attack space to address such computational issues in Subsection \ref{sec:relax}. This conservative relaxation where the attacker is viewed to possess more power than in practice enables us to transform the problem into an efficient linear program (LP) with substantially smaller size. In addition to game theoretical results, we also analyze the performance of the proposed detection mechanism numerically over various scenarios. 

Our main contributions are as follows:
\begin{itemize}
\item To the best of our knowledge, this is the first work in the literature to address adversarial intervention on smart road signs within a game theoretical framework. 
\item The randomized detection rule developed under the solution concept of game theoretical equilibrium ensures robustness against the worst-case attacker that attacks to maximize the cost for the detector while knowing the detection mechanism.
\item By relaxing the attacker's strategy space, we provide an efficient (offline) LP-based algorithm to compute the best randomized detection strategy, which can reduce the verification complexity. 
\end{itemize}

The paper is organized as follows: In Section \ref{sec:prelim}, we provide preliminary information about error-correction coding. In Section \ref{sec:problem}, we formulate the problem. In Section \ref{sec:game}, we analyze the equilibrium of the game. We provide numerical examples in Section \ref{sec:examples}. We conclude the paper and identify possible research directions in Section \ref{sec:conclusion}. 

\section*{Nomenclature}
\addcontentsline{toc}{section}{Nomenclature}
\noindent
{\em Problem Setting:}
\begin{IEEEdescription}[\IEEEusemathlabelsep\IEEEsetlabelwidth{$\loss:\trueSymbols\times\symbolsSet\rightarrow \{0,1\}$}]
\item[${[n,k,d]}_q$] linear block code
\item[$n\in\Z$] codeword length
\item[$k\in\Z$] message length
\item[$d\in\Z$] (minimum) distance of the code
\item[$d_o = \lfloor (d-1)/2\rfloor$] error-correction diameter
\item[$\Sigma$] alphabet of the symbols
\item[$q=|\Sigma|$] alphabet size
\item[$\symbolsSet$] set of all codewords
\item[$\trueSymbols\subset\symbolsSet$] set of encoded codewords
\item[$\decodableSymbols\subset\symbolsSet$] set of decodable codewords
\item[$\trueSymbol\in\trueSymbols$] encoded codeword
\item[$\noisySymbol\in\noisySet$] received (noisy) codeword
\item[$\decoderOutput\in\trueSymbols$] decoder output for decodable $\noisySymbol$
\item[$\channel$] noisy channel
\item[$p(\trueSymbol),\trueSymbol\in\trueSymbols$] probability of $\trueSymbol\in\trueSymbols$
\item[$\errorProb\in{[0,1]}$] probability of error in a symbol
\item[$\Hamming:\symbolsSet\times\symbolsSet\rightarrow \Z$] Hamming distance 
\end{IEEEdescription}
{\em Game Setting:}
\begin{IEEEdescription}[\IEEEusemathlabelsep\IEEEsetlabelwidth{$\loss:\trueSymbols\times\symbolsSet\rightarrow \{0,1\}$}]
\item[$\game$] road-sign classification game
\item[\player{A}] Attacker
\item[\player{D}] Detector
\item[$U(\Amixed,\Dmixed)$] \player{D}'s cost function
\item[$\Dmixed\in\DmixedSpace$] \player{D}'s randomized detection rule
\item[$\Aactiono\in\trueSymbols$] attacked codeword
\item[$\Aactiona\in\symbolsSet$] crafted codeword
\item[$\Aaction=(\Aactiono,\Aactiona)$] \player{A}'s (pure) action
\item[$\AactionSpace=\trueSymbols\times\symbolsSet$] \player{A}'s action space
\item[$\Amixed = \{\Amixed_{\Aaction}\} \in \AmixedSpace$] \player{A}'s mixed strategy
\item[$\loss:\trueSymbols\times\symbolsSet\rightarrow \{0,1\}$] loss due to decoding error/failure
\item[$\factor{j}\geq 0, j=1,\ldots,4$] multiplicative factors
\end{IEEEdescription}

\begin{figure*}
\centering
\begin{overpic}[width = \textwidth]{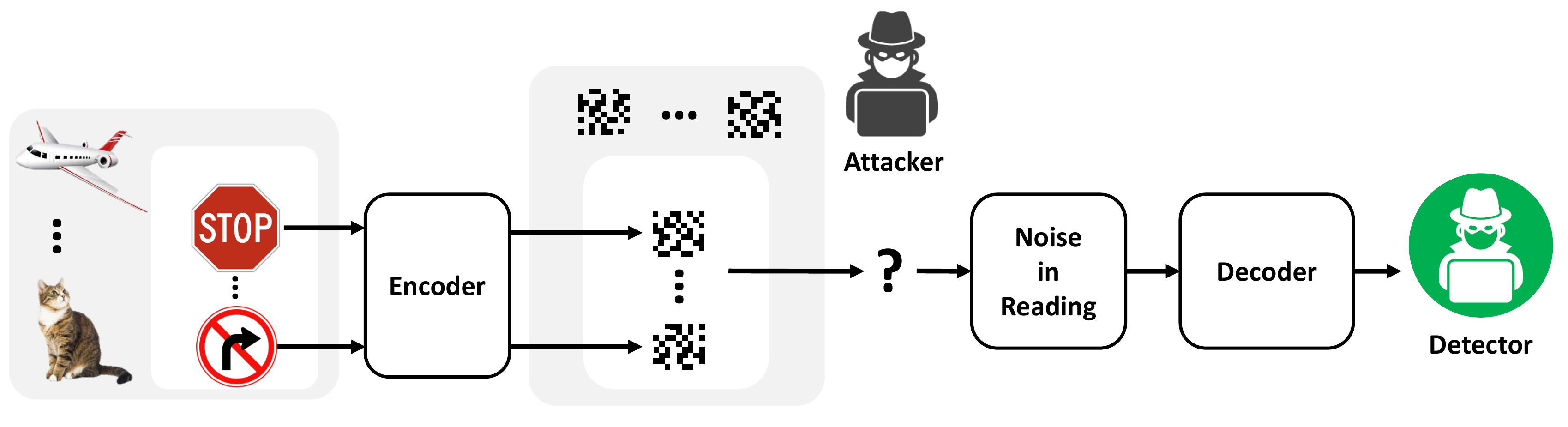}
\put(181,125){Codeword Space $\symbolsSet$}
\put(54,82){Road Signs}
\put(196,78){Smart Codes}
\put(282,63){\Huge \ding{45}}
\end{overpic}
\caption{A figurative illustration of the interaction in-between \player{A} and \player{D}. \player{A} selects which road sign to attack and attacks its codeword by changing the symbols physically as exemplified in Fig. \ref{fig:attack}. \player{D} observes a noisy version of the codeword and seeks to detect whether there has been an attack or not.} 
\label{fig:model}
\end{figure*}

\section{Preliminaries in Error Correction Coding}\label{sec:prelim}

Error correction codes provide certain formal guarantees for the transmission of digital data over noisy channels as long as the deviation on the message sent due to random/intelligent noise is less than a certain threshold with respect to a certain distance metric \cite{ref:Blahut02}.  Since a smart code can be viewed as a finite-size block and perturbations can be viewed as Boolean operations, e.g., flipped bits, we specifically consider linear block codes, which encode data in blocks. They are called linear because any linear combination of codewords is also a codeword. Formally, a linear block code, denoted by $[n,k,d]_q$, operates over a finite alphabet of symbols whose size is denoted by $q\in\Z$, and maps $k\in\Z$ symbols to $n\in\Z$ symbols. The (minimum) distance of a block, denoted by $d\in\Z$, is the minimum number of positions where any two distinct codewords differ, i.e., the Hamming distance \cite{ref:Blahut02} in-between the distinct codewords. The abstraction of the code via $[n,k,d]_q$ enables us to study all the linear block codes in a unified way. 

We emphasize that the Singleton bound \cite{ref:Singleton64, ref:Roman92} that all linear block codes satisfy is given by 
\begin{equation}\label{eq:Singleton}
d\leq n-k+1,
\end{equation}
where the equality holds for Reed-Solomon codes \cite{ref:Reed60}. Furthermore, the minimum distance $d$ implies that the block code can detect $d-1$ symbol errors and correct up to 
\begin{equation}\label{eq:do}
d_o := \left\lfloor \frac{d-1}{2} \right\rfloor
\end{equation}
symbol errors since there exists no other codeword within $d-1$ diameter of each codeword. 

Consider that the number of symbol errors, denoted by $e\in\Z$, is more than half of the minimum distance, i.e., $e > d_o$. We say that a {\em decoding error} exists if the Hamming distance between the received codeword and any other codeword is less than or equal to $d_o$, i.e., if we decode it erroneously. Further, we say that a {\em decoding failure} exists if the Hamming distance between the received codeword and all the other codewords is more than $d_o$, i.e., if the received codeword is not decodable \cite{ref:daraiseh98}.

\section{Problem Formulation}\label{sec:problem}

Consider that road signs are encoded into smart codes via a linear block code $[n,k,d]_q$, and there exist two players: an attacker (\player{A}) and a detector (\player{D}), as seen in Fig. \ref{fig:model}. Given the encoding-decoding scheme and the underlying statistical profiles, \player{D} seeks to detect any intervention by \player{A} while \player{A} seeks to modify the smart codes physically, as exemplified in Fig. \ref{fig:attack}, in order to lead to decoding error/failure stealthily. 

{\em Noise Model.} The decoder reads a noisy version of the smart code due to, e.g., lighting-induced blurring or harsh weather conditions. Let $\symbolsSet$ denote the codeword space. Then, we model this noise via a probability transition mapping $p(\noisySymbol | \justSymbol)$ corresponding to the probability of receiving codeword $\noisySymbol\in\noisySet$ given that the transmitted codeword is $\justSymbol\in\symbolsSet$. We suppose that all the symbol errors by nature are equally likely and independent of each other. We denote the probability that there can be an error in a symbol by $\errorProb\in[0,1]$. We also suppose that the change of the symbol to any other symbol in the alphabet is equally likely in a symbol error. 

\begin{remark}[Symbol Error]
In a codeword, a symbol consists of multiple contiguous bits. A symbol error occurs if at least one of these bits is perturbed. Correspondingly, if random perturbations infect multiple contiguous bits, the number of symbol errors is an appropriate distance measure. This is indeed the case in smart road signs due to possible obfuscation by plants or graffiti or adversarial stickers as studied in \cite{ref:Eykholt18} or as illustrated in Fig. \ref{fig:attack}. When there are blurring due to lighting throughout a day, fading colors, or weather conditions, we would also expect perturbations on multiple contiguous bits, instead of perturbations on single isolated bits (which may require surgical-like precision due to the relatively small size of a single bit). Furthermore, error-correction codes, e.g., Reed-Solomon codes, provide effective guarantees against symbol errors. 
\end{remark}

{\em Defense Model.} \player{D} has multiple objectives:
\begin{itemize}
\item[$O1)$] to minimize the cost of losing the opportunity to prevent future attacks by not being able to detect it now,
\item[$O2)$] to minimize the cost associated with adversary-induced decoding error/failure,
\item[$O3)$] to minimize the cost associated with false alarms,
\item[$O4)$] to maximize the number of symbol errors necessary to deceive the decoder.
\end{itemize}
We model \player{D}'s cost function as a linear combination of these objectives with certain multiplicative factors, which give flexibility to control the weight of the corresponding objective as desired. This cost function will be defined explicitly in the game model. 

We let $\trueSymbols\subset\symbolsSet$ denote the set of encoded codewords, i.e., there exists bijective relation in-between $\trueSymbols$ and the set of road signs. We also let $\decodableSymbols\subset\symbolsSet$ denote the set of decodable codewords that are within $d_o$ diameter of a codeword in $\trueSymbols$. If $\noisySymbol\in\decodableSymbols$, we denote the decoder output by $\decoderOutput\in\trueSymbols$. The loss due to decoding error/failure is given by
\begin{equation}\label{eq:loss}
\loss(\trueSymbol,\noisySymbol) = \left\{\begin{array}{ll} 0 & \mbox{if } \noisySymbol \in \decodableSymbols \mbox{ and } \decoderOutput=\trueSymbol \\ 1 & \mbox{otherwise.}  \end{array}\right.
\end{equation} 

If $\noisySymbol\in\decodableSymbols$, \player{D} can report an issue against the possibility of adversarial intrusion so that further (costly) investigations can take place. To this end, \player{D} designs a randomized detection rule $\Dmixed\in\DmixedSpace$, where $\Dmixed_j$, $j=0,\ldots,d_o$, corresponds to the probability of triggering an alert for $j$ symbol errors. If $\noisySymbol\notin\decodableSymbols$, further investigations always take place.
  
\begin{remark}[Scalable Defense]
We consider a randomized detection rule depending on the number of symbol errors for scalability. If \player{D} were to select a (randomized) detection rule based on the received codeword, then \player{D} would select a vector over the space $[0,1]^{|\symbolsSet|}$, which is $q^n$ dimensional.  Note that $q^n$ is exponential in the number of symbols $n$ whereas $d_o+1 \ll q^n$ is linear in $n$.\end{remark}

{\em Threat Model.} \player{A} is the worst case attacker who maximizes \player{D}'s cost function. To this end, \player{A} can select which road sign to attack. Let $\Aactiono\in \trueSymbols$ denote the codeword of the attacked road sign. Then, \player{A} can craft $\Aactiono\in\trueSymbols$ to $\Aactiona\in\symbolsSet$ by introducing error in order to control the decoder output. The complexity of this crafting is given by the number of symbols changed, i.e., $\Hamming(\Aactiono,\Aactiona)$. We denote \player{A}'s action space by $\AactionSpace:=\trueSymbols\times\symbolsSet$ and denote \player{A}'s action by $\Aaction:=(\Aactiono,\Aactiona)$. \player{A} can select a mixed strategy $\Amixed = \{\alpha_{\Aaction}\}$ over $\AactionSpace$ such that $\Amixed_{\Aaction}$ denotes the probability of taking action $\Aaction=(\Aactiono,\Aactiona)\in\AactionSpace$, i.e., attacking the codeword $\Aactiono\in\trueSymbols$ and crafting it to $\Aactiona\in\symbolsSet$. 

{\em Game Model.} We consider a zero-sum game setting where \player{D} seeks to minimize the cost function:
\begin{align}\label{eq:cost}
U(\Amixed,\Dmixed) = \,&\factor{1}\sum_{\Aaction \in \AactionSpace}\Amixed_{\Aaction} \Bigg(\sum_{j=0}^{d_o}(1-\Dmixed_j) \hspace{-.1in}\sum\limits_{\substack{\noisySymbol\in\decodableSymbols\\ \ni \Hamming(\noisySymbol,\decoderOutput)=j}}\hspace{-.1in}p(\noisySymbol|\Aactiona)\Bigg)\nn\\
+&\factor{2}\sum_{\Aaction \in \AactionSpace}\Amixed_{\Aaction}\sum_{\noisySymbol\in\noisySet}\loss(\Aactiono,\noisySymbol)p(\noisySymbol|\Aactiona)\nn\\
+&\factor{3} \sum_{j=0}^{d_o}\Dmixed_j \sum_{\trueSymbol\in\trueSymbols}\hspace{-.1in}\sum\limits_{\substack{\noisySymbol\in\decodableSymbols\\ \ni \Hamming(\noisySymbol,\trueSymbol)=j}}\hspace{-.1in}p(\noisySymbol | \trueSymbol)p(\trueSymbol)\nn\\
-&\factor{4} \sum_{\Aaction \in \AactionSpace} \Amixed_{\Aaction} \Hamming(\Aactiono,\Aactiona)
\end{align}
against the worst-case \player{A} who seeks to maximize \eqref{eq:cost}. We define certain multiplicative factors $\factor{j}\geq0$, $j=1,2,3,4$, corresponding, respectively, to \player{D}'s objectives $O1)-O4)$. Note that minimization of the expected cost due to the uncertainty of the channel $\channel$ is in-line with the expectation-over-transformation framework proposed in \cite{ref:Athalye18b}. The attackers can generate robust attacks by considering the expected impact of the uncertainties due to the channel \cite{ref:Athalye18b}. 

\begin{remark}[Attack Probability]
If \player{D} has a priori information $p_a\in[0,1]$ corresponding to the probability of adversarial intervention, then we can incorporate this into \eqref{eq:cost} by selecting the multiplicative factors accordingly. For example, the objectives $O1)$, $O2)$, and $O4)$ matter if there is an adversarial intervention while the objective $O3)$ matters if there is no adversarial intervention. To this end, we can scale up $\factor{1}$, $\factor{2}$, and $\factor{4}$ by $p_a$ while scaling up $\factor{3}$ by $1-p_a$.
\end{remark}

We consider a hierarchical setting, where \player{A} can know (or learn) \player{D}'s randomized detection algorithm, in order to avoid obscurity based defense, which can be bypassed when an advanced attacker learns the information in obscurity. Therefore, this interaction can be modeled as a Stackelberg zero-sum game\footnote{$\Delta^{d-1}\subset\R^{d}$ denotes the probability simplex formed by $d$ standard unit vectors.}
\begin{equation}
\game := (\AmixedSpace,\DmixedSpace,U,\loss(\cdot),\channel,p(\trueSymbol),\{\factor{j}\}),
\end{equation}
where \player{D} is the leader and \player{A} is the follower. Since \player{A} is the follower and reacts to \player{D}'s strategy $\Dmixed\in\DmixedSpace$, the problem faced by the detector is given by 
\begin{align}\label{eq:SE}
\min_{\Dmixed\in\DmixedSpace}\max_{\Amixed\in\AmixedSpace} U(\Amixed,\Dmixed).
\end{align}

The following proposition shows that there exists an equilibrium to the game $\game$.

\begin{proposition}[Existence Result]\label{prop:existence}
There exists a pair of \player{D}'s strategy and \player{A}'s reaction $(\Dmixed^*,\bestResponse(\Dmixed^*))$ attaining the Stackelberg equilibrium $\game$, i.e., satisfying \eqref{eq:SE}.
\end{proposition}

\begin{proof}
Note that $U(\Amixed,\Dmixed)$ is linear, and correspondingly, continuous in the optimization arguments $\Amixed\in\AmixedSpace$ and $\Dmixed\in\DmixedSpace$, and the constraint sets are decoupled. Therefore the maximum theorem \cite{ref:Ok07} yields that
\begin{equation}\label{eq:Areact}
\max_{\Amixed\in\AmixedSpace} U(\Amixed,\Dmixed)
\end{equation}
is a continuous function of $\Dmixed\in\DmixedSpace$. Then, since $\DmixedSpace$ is a compact set, the extreme value theorem yields that there exits a solution for \eqref{eq:SE}. 
\end{proof}

In the following section, we analyze the equilibrium to $\game$.

\section{Adversarial Intervention Detection Across Smart Road Signs} \label{sec:game}

Existence of an equilibrium is guaranteed as shown in Proposition \ref{prop:existence}. However, computation of the equilibrium can be demanding (even if it is an offline computation) since \player{A} has a large strategy space, even for short codewords. In this section, our goal is to examine \player{A}'s best response for efficient computation of the best detection rule. To this end, we first seek to formulate certain equivalence classes over \player{A}'s actions such that all the actions in a class lead to the same outcome of the game (see, Subsection \ref{sec:equivalence}). However, depending on the size of the input space, i.e., $\symbolsSet$, computation of the equilibrium may still be demanding for that quotient space. In order to avoid such a computational issue for long codewords that can express relatively larger collection of road signs, we relax the constraints on \player{A}'s action space, which will lead to more powerful attacker than in practice (see, Subsection \ref{sec:relax}). This yields a conservative defense, which leads to lower cost against the actual attacker who is relatively less powerful in run-time applications. Finally, we transform the problem into an efficient LP, rather routinely, in order to apply existing powerful computational tools to compute the best detection rule. We now provide the details of these steps.    

Our goal is to compute the best detection rule $\Dmixed^*\in\DmixedSpace$ with respect to the equilibrium \eqref{eq:SE}. To this end, \player{D} needs to anticipate \player{A}'s reaction to any selected detection rule. However, \player{A}'s action space $\AactionSpace$ has dimension $|\AactionSpace| = q^{n+k}$, which is exponential in $n+k$. Therefore, finding the best reaction, i.e., a vector in $\AmixedSpace$, for each detection rule is computationally demanding. Accordingly, in the following, we seek to reduce \player{A}'s strategy space without losing the generality.

\subsection{Equivalence Classes on \player{A}'s Best Response}\label{sec:equivalence}

\begin{figure}[t!]
\centering
\begin{tikzpicture}

\def\scale{1.3}
\def\L{6*\scale}
\def\l{5*\scale}
\def\lo{\l/5}
\def\loo{2*\lo}
\def\looo{3*\lo}
\def\loooo{4*\lo}
\def\lshift{\lo/10}
\def\xshift{\lo/5}
\def\Ax{0}
\def\Ay{0}
\def\Bx{\Ax-\l*cos(74)}
\def\By{\Ay+\l*sin(74)}
\def\Cx{\Ax+\L*cos(53)}
\def\Cy{\Ay+\L*sin(53)}
\def\shifto{17.5}
\def\shiftoo{9}
\def\shiftooo{6}
\def\shiftoooo{4.5}

\fill[myorange] ({\Ax},{\Ay}) circle(.03in);
\fill[myorange] ({\Ax+\looo*cos(53)},{\Ay+\looo*sin(53)}) circle(.03in);

\draw[line width = .2mm,dotted] ({\Ax},{\Ay}) node[left] {$\trueSymbol^1$} -- 
                                        ({\Bx},{\By}) node[above] {$\trueSymbol^2$} -- 
                                        ({\Cx},{\Cy}) node[above] {$\trueSymbol^3$} -- 
                                        ({\Ax},{\Ay}) {};

\draw[pattern=north west lines, pattern color=myblue,draw=none] ({\Bx},{\By}) -- ({\Bx+\loo},{\By}) arc[radius = \loo, start angle = 0, end angle=-74] -- ({\Bx},{\By});
\draw[draw=myblue,line width = .2mm] ({\Bx+\lo*cos(\shifto)},{\By+\lo*sin(\shifto)}) arc[radius = \lo, start angle = \shifto, end angle=-74-\shifto];
\draw[draw=myblue,line width = .2mm] ({\Bx+\loo*cos(\shiftoo)},{\By + \loo*sin(\shiftoo)}) arc[radius = \loo, start angle = \shiftoo, end angle=-74-\shiftoo];
\draw[draw=myblue,line width = .2mm] ({\Bx+\looo*cos(\shiftooo)},{\By+\looo*sin(\shiftooo)}) arc[radius = \looo, start angle = \shiftooo, end angle=-74-\shiftooo];

\draw[pattern=north west lines, pattern color=myblue,draw=none] ({\Cx},{\Cy}) -- ({\Cx-\loo},{\Cy}) arc[radius = \loo, start angle = 180, end angle=233] -- ({\Cx},{\Cy});
\draw[draw=myblue,line width = .2mm] ({\Cx-\lo*cos(\shifto)},{\Cy+\lo*sin(\shifto)}) arc[radius = \lo, start angle = 180-\shifto, end angle=233+\shifto];
\draw[draw=myblue,line width = .5mm,line join=round,
decorate,decoration={coil,amplitude=1.5}] ({\Cx-\loo*cos(\shiftoo)},{\Cy+\loo*sin(\shiftoo)})  arc[radius = \loo, start angle = 180-\shiftoo, end angle=233+\shiftoo] node[right] {Choice-$3$};
\draw[draw=myblue,line width = .2mm] ({\Cx-\looo*cos(\shiftooo)},{\Cy+\looo*sin(\shiftooo)}) arc[radius = \looo, start angle = 180-\shiftooo, end angle=233+\shiftooo];

\draw[<-,line width = .3mm,myblue!50!black] ({\Cx-(\lo+2*\lshift)*cos(0)},{\Cy-(\lo+2*\lshift)*sin(0)-\xshift}) -- ({\Cx-(\loo-2*\lshift)*cos(0)},{\Cy-(\loo-2*\lshift)*sin(0)-\xshift});
\draw[->,line width = .5mm,myblue!50!black] ({\Cx-(\loo+2*\lshift)*cos(0)},{\Cy-(\loo+2*\lshift)*sin(0)-\xshift}) -- ({\Cx-(\looo-2*\lshift)*cos(0)},{\Cy-(\looo-2*\lshift)*sin(0)-\xshift});
\draw[<-,line width = .3mm,myblue!50!black] ({\Cx-(\lo+\lshift)*cos(53)-\xshift},{\Cy-(\lo+\lshift)*sin(53)}) -- ({\Cx-(\loo-3*\lshift)*cos(53)-\xshift},{\Cy-(\loo-3*\lshift)*sin(53)});
\draw[->,line width = .5mm,myblue!50!black] ({\Cx-(\loo+1*\lshift)*cos(53)-\xshift},{\Cy-(\loo+1*\lshift)*sin(53)}) -- ({\Cx-(\looo-3*\lshift)*cos(53)-\xshift},{\Cy-(\looo-3*\lshift)*sin(53)});
\draw[<->,line width = .4mm,myblue!50!black] ({\Cx-(\loo-2*\lshift)*cos(15)},{\Cy-(\loo-2*\lshift)*sin(15)}) arc[radius=\loo-2*\lshift,start angle=195, end angle=218];

\draw[pattern=north west lines, pattern color=myblue,draw=none] ({\Ax},{\Ay}) node[right] {~~Choice-$0$} -- ({\Ax+\loo*cos(53)},{\Ay+\loo*sin(53)}) arc[radius = \loo, start angle = 53, end angle=106] -- ({\Ax},{\Ay});    
\draw[->,line width = .5mm] ({\Ax-(\lo/2)*cos(20)},{\Ay+(\lo)*sin(20)}) node[left]{Decodable Region} -- ({\Ax-cos(74+\shifto)*\lo/2},{\Ay+sin(74+\shifto)*\lo/2});                                    
\draw[draw=myorange,line width = .2mm] ({\Ax+\lo*cos(53-\shifto)},{\Ay+\lo*sin(53-\shifto)}) arc[radius = \lo, start angle = 53-\shifto, end angle=106+\shifto] node[left] {$1$ error};
\draw[draw=myorange,line width = .5mm,line join=round,
decorate,decoration={coil,amplitude=1.5}] ({\Ax+\loo*cos(53-\shiftoo)},{\Ay+\loo*sin(53-\shiftoo)}) node[right] {Choice-$1$} arc[radius = \loo, start angle = 53-\shiftoo, end angle=106+\shiftoo] node[left] {$2$ errors};                                        
\draw[draw=myorange,line width = .2mm] ({\Ax+\looo*cos(53-\shiftooo)},{\Ay+\looo*sin(53-\shiftooo)}) node[right] {Choice-$2$} arc[radius = \looo, start angle = 53-\shiftooo, end angle=106+\shiftooo]  node[left] {$3$  errors};
\draw[draw=myorange,line width = .2mm] ({\Ax+\loooo*cos(53-\shiftoooo)},{\Ay+\loooo*sin(53-\shiftoooo)}) arc[radius = \loooo, start angle = 53-\shiftoooo, end angle=106+\shiftoooo] node[left] {$4$ errors};

\draw[<-,line width = .3mm,myorange!50!black] ({\Ax+(\lo+3*\lshift)*cos(53)-\xshift},{\Ay+(\lo+3*\lshift)*sin(53)}) -- ({\Ax+(\loo-\lshift)*cos(53)-\xshift},{\Ay+(\loo-\lshift)*sin(53)});
\draw[<-,line width = .3mm,myorange!50!black] ({\Ax-(\lo+2*\lshift)*cos(74)+\xshift},{\Ay+(\lo+2*\lshift)*sin(74)}) -- ({\Ax-(\loo-2*\lshift)*cos(74)+\xshift},{\Ay+(\loo-2*\lshift)*sin(74)});
\draw[->,line width = .5mm,myorange!50!black] ({\Ax+(\loo+3*\lshift)*cos(53)-\xshift},{\Ay+(\loo+3*\lshift)*sin(53)}) -- ({\Ax+(\looo-\lshift)*cos(53)-\xshift},{\Ay+(\looo-\lshift)*sin(53)});
\draw[->,line width = .5mm,myorange!50!black] ({\Ax-(\loo+2*\lshift)*cos(74)+\xshift},{\Ay+(\loo+2*\lshift)*sin(74)}) -- ({\Ax-(\looo-2*\lshift)*cos(74)+\xshift},{\Ay+(\looo-2*\lshift)*sin(74)});
\draw[<->,line width = .4mm,myorange!50!black] ({\Ax-(\loo-2*\lshift)*cos(89)},{\Ay+(\loo-2*\lshift)*sin(89)}) arc[radius=\loo-2*\lshift,start angle=91, end angle=68];

\end{tikzpicture}
\vspace{.1in}
\caption{Figurative illustration of $\symbolsSet$ for the Reed-Solomon Code $[7,3,5]_q$. Suppose the attacked codeword is $\Aactiono = \trueSymbol^1$. Decodable regions for the codewords $\trueSymbol^1,\trueSymbol^2,\trueSymbol^3\in\trueSymbols$ are shaded and arcs correspond to the levels of symbol errors. The color coded arrows illustrate figuratively how the corresponding level of symbol error would change with additional nature-induced noisy perturbation of the crafted codeword.}
\label{fig:levels}
\end{figure}

Fig. \ref{fig:levels} provides a figurative illustration of how \player{A} can attack. Note that error-correction methods ensure that different encoded codewords are at least a certain number of symbols away from each other. For the code $[7,3,5]_q$ in Fig. \ref{fig:levels}, this minimum distance is $d = 5$ symbols as exemplified in between $\trueSymbol^1$ and $\trueSymbol^2$. Any perturbation that can change at most $d_o = 2$ symbols does not lead to any decoding error or failure. However, at certain directions, perturbations on $3$ symbols can lead to a decoding error, e.g., by carrying $\trueSymbol^1$ to the decodable region of $\trueSymbol^2$, or a decoding failure depending on how and which symbols are perturbed. Correspondingly, if \player{A} decides to attack on $\Aactiono=\trueSymbol^1$, then \player{A} has several choices while physically damaging the corresponding smart code. In the following, we categorize those choices into four main groups: 
\begin{itemize}
\item[$C0)$] \player{A} may not attack, i.e., may not introduce any error.
\item[$C1)$] \player{A} may introduce relatively smaller amount of symbol error(s) such that the corrupted codeword is still in the decodable region of the associated codeword. Due to the channel, this can still lead to decoding error or failure with certain probabilities.
\item[$C2)$] \player{A} may introduce symbol errors such that the corrupted codeword becomes not decodable. 
\item[$C3)$] \player{A} may introduce relatively larger amount of symbol errors such that the codeword intervened is in the decodable region of another codeword. 
\end{itemize}
Even if there were not any detection rule, in Fig. \ref{fig:levels}, we observe that not attacking, i.e., $C0)$, or attacking relatively more aggressively, e.g., $C2)$ and $C3)$, are not necessarily more preferable for \player{A} than $C1)$ due to the noisy perturbations and the trade-off between the amount of perturbation and the gain of \player{A} by decoding error/failure. In the following, we examine the channel, which can lead to such intriguing results. 

Recall that any symbol can be perturbed by the channel with the same probability $\errorProb$ while the symbol perturbed can change to any other symbol in the alphabet with the same probability, i.e., $1/(q-1)$. Then, the probability that $\justSymbol$ turns into $\noisySymbol$ due to noisy channel can be written as
\begin{equation}\label{eq:probTrans}
p(\noisySymbol | \justSymbol) = (1-\errorProb)^{n-\Hamming(\noisySymbol,\justSymbol)}\left(\frac{\errorProb}{q-1}\right)^{\Hamming(\noisySymbol,\justSymbol)},
\end{equation}
which only depends on the distance in-between $\noisySymbol\in\symbolsSet$ and $\justSymbol\in\symbolsSet$. Particularly, there are $n-\Hamming(\noisySymbol,\justSymbol)$ symbols that match at $\justSymbol$ and $\noisySymbol$. There should not be any perturbations on those symbols, which leads to the first multiplicative term on the right-hand-side of \eqref{eq:probTrans}. For each symbol that does not match, the random perturbation must change the one at $\justSymbol$ to the corresponding one at $\noisySymbol$ among $q-1$ equally likely alternatives, which leads to the second multiplicative term. 

Since $p(\noisySymbol | \justSymbol)$ only depends on $\Hamming(\noisySymbol,\justSymbol)$, we define an auxiliary metric $\rho:\Z\times\Z \rightarrow [0,1]$, where $\rho(n_1,n_2)$ denotes the probability that two codewords that are $n_1\in\Z$ symbols away become $n_2\in\Z$ symbols away when one of them is randomly perturbed by the channel. Note that we can compute $\rho(\cdot)$ based on combinatorics analytically or using the Monte Carlo method \cite{ref:Kroese14} numerically. With this new auxiliary metric, let us take a closer look into the objectives $O1)$ and $O2)$, where the channel and \player{A} have impact on. Firstly, the term in parenthesis in $O1)$ can be written as
\begin{align}
\sum_{j=0}^{d_o}(1-\Dmixed_j) \hspace{-.1in}\sum\limits_{\substack{\noisySymbol\in\decodableSymbols\\ \ni \Hamming(\noisySymbol,\decoderOutput)=j}}\hspace{-.1in}p(\noisySymbol|\Aactiona&) = \sum_{j=0}^{d_o} (1-\Dmixed_j) \underbrace{\sum_{\trueSymbol\in\trueSymbols}\rho(\Hamming(\trueSymbol,\Aactiona),j)},\label{eq:piTp}
\end{align}
where the under-braced term corresponds to the total probability that $\Aactiona$ moves to $j$ symbols away from an encoded codeword due to the random noise. Similarly, we can write the inner summation in $O2)$ as
\begin{align}
\sum_{\noisySymbol\in\symbolsSet} \loss(\Aactiono,\noisySymbol) p(\noisySymbol | \Aactiona) &= \sum_{\noisySymbol\in\symbolsSet} \indicator_{\{\Hamming(\Aactiono,\noisySymbol)>d_o\}} p(\noisySymbol|\Aactiona)\\
&= 1- \sum_{t=0}^{d_o} \rho(\Hamming(\Aactiono,\Aactiona),t),\label{eq:hloss}
\end{align}
where the first line follows since a detection error or failure occurs if the received codeword $\noisySymbol\in\symbolsSet$ is more than $d_o\in\Z$ symbols away from the codeword $\Aactiono\in\trueSymbols$ before \player{A} crafts it into $\Aactiona\in\symbolsSet$. The second line follows by the definition of the new auxiliary metric $\rho(\cdot)$.

Note that $O1)$ written according to \eqref{eq:piTp} depends on the distance between $\Aactiona\in\symbolsSet$ and all the encoded codewords $\trueSymbol\in\trueSymbols$. Similarly, only the distance between $\Aactiono\in\trueSymbols$ and $\Aactiona\in\symbolsSet$ has an impact on $O2)$ and $O4)$ while we also have $\Aactiono \in \trueSymbols$. Therefore, the attacks that target $\Aactiono$ have the same impact on the cost function \eqref{eq:cost} if the distances between $\Aactiona$ and the encoded codewords are the same. 

Indeed, there is a strong coupling on how \player{A} would select $\Aactiono$ and $\Aactiona$ independent of \player{D}'s strategy. Particularly, \eqref{eq:piTp}, and correspondingly $O1)$, do not include $\Aactiono\in\trueSymbols$. On the other side, the objectives $O2)$ and $O4)$, which include $\Aactiono\in\trueSymbols$, do not include $\Dmixed\in\DmixedSpace$. Therefore for given $\Aactiona\in\symbolsSet$, \player{A} can select $\Aactiono\in\trueSymbols$ irrespective of \player{D}'s detection rule. Based on this observation in the following lemma, we eliminate weakly dominated actions of \player{A} in order to reduce the size of \player{A}'s strategy space.

\begin{lemma}
In the game $\game$, without loss of generality, we can restrict \player{A}'s action space $\AactionSpace$ into 
\begin{align}
\AactionSpace^d = \{(\Aactiono^*,\Aactiona\in\symbolsSet)\}, 
\end{align}
where $\Aactiono^*$ is the maximizer of the optimization problem:
\begin{align}
\max_{\Aactiono\in\trueSymbols} \factor{2}\Big(1 -  \sum_{t=0}^{d_o} \rho(\Hamming(\Aactiono,\Aactiona),t)\Big) - \factor{4} \Hamming(\Aactiono,\Aactiona).\label{eq:Aoopt}
\end{align}
\end{lemma}

\begin{proof}
The terms corresponding to $O2)$ and $O4)$ in \eqref{eq:cost} can be written as
\begin{align}
&\hspace{-.13in}\factor{2}\sum_{\Aaction \in \AactionSpace}\Amixed_{\Aaction}\sum_{\noisySymbol\in\noisySet}\loss(\Aactiono,\noisySymbol)p(\noisySymbol|\Aactiona) - \factor{4} \sum_{\Aaction \in \AactionSpace} \Amixed_{\Aaction} \Hamming(\Aactiono,\Aactiona) \nn\\
=& \sum_{\Aaction \in \AactionSpace}\Amixed_{\Aaction}\bigg(\factor{2}\Big(1 -  \sum_{t=0}^{d_o} \rho(\Hamming(\Aactiono,\Aactiona),t)\Big) - \factor{4} \Hamming(\Aactiono,\Aactiona)\bigg), 
\end{align}
which follows by \eqref{eq:hloss}. Since \player{A} seeks to maximize the cost \eqref{eq:cost}, for each $\Aactiona\in\symbolsSet$, we can compute the associated optimal attacked codeword, i.e., $\Aactiono^*$, via \eqref{eq:Aoopt}, where a solution is guaranteed to exists since the constraint set $\trueSymbols$ is finite. 
\end{proof}

Note also that \eqref{eq:piTp}, \eqref{eq:hloss}, and \eqref{eq:Aoopt} depend only on the distances between $\Aactiona\in\symbolsSet$ and all the encoded codewords. Therefore, for a given detection rule, any other $\tilde{a}_x\in\symbolsSet$ that has the same set of distances to the encoded codewords would lead to the same cost \eqref{eq:cost}. Correspondingly, we define another auxiliary function $\allHamming:\symbolsSet \rightarrow \Z^{|\trueSymbols|}$ such that $\allHamming(\Aactiona)$ is a vector whose $t$th entry, denoted by $\allHamming_t(\Aactiona)$, corresponds to the distance in-between $\Aactiona$ and the $t$th encoded codeword (with respect to a certain order in $\trueSymbols$). Then, given $\allHamming(\Aactiona)$, \eqref{eq:Aoopt} can be written as
\begin{align}
\reward(\allHamming(\Aactiona)) := \max_{\hbar\in\{\allHamming(\Aactiona)\}} \factor{2} - \factor{4} \hbar - \factor{2}\sum_{t=0}^{d_o} \rho(\hbar,t),\label{eq:Aoopt2}
\end{align}
where $\{\allHamming(\Aactiona)\}$ denotes the set including the entries of the vector $\allHamming(\Aactiona)$. We can view $\reward(\allHamming(\Aactiona))$ as the reward of \player{A} for $\allHamming(\Aactiona)$. Therefore, by \eqref{eq:piTp}, \eqref{eq:hloss}, and \eqref{eq:Aoopt2}, the cost function $U(\Amixed,\Dmixed)$ can be written as
\begin{align}
U(\Amixed,\Dmixed) =\,&\factor{1}\sum_{\Aaction\in\AactionSpace^d} \Amixed_{\Aaction}\bigg(\sum_{j=0}^{d_o} (1-\Dmixed_j)\sum_{t=1}^{|\trueSymbols|}\rho(\allHamming_t(\Aactiona),j)\bigg)\nn\\
+&\factor{3} \sum_{j=0}^{d_o}\Dmixed_j \sum_{\trueSymbol\in\trueSymbols}\hspace{-.1in}\sum\limits_{\substack{\noisySymbol\in\decodableSymbols\\ \ni \Hamming(\noisySymbol,\trueSymbol)=j}}\hspace{-.1in}p(\noisySymbol | \trueSymbol)p(\trueSymbol)\nn\\
+&\sum_{\Aaction\in\AactionSpace^d} \Amixed_{\Aaction}\reward(\allHamming(\Aactiona)).\label{eq:cost2}
\end{align}

The following lemma recaps these results to formulate the equivalence classes on \player{A}'s best response.

\begin{lemma}\label{lem:equivalence}
Without loss of generality, instead of mixing over $\AactionSpace$, \player{A} can select a mixed strategy across the quotient set $\AactionSpace^d/\sim$ with respect to the following equivalence relation:
\begin{align}\label{eq:equivalence}
(\Aactiono^*,\Aactiona) \sim (\tilde{a}_o^*,\tilde{a}_x) \Leftrightarrow \allHamming(\Aactiona) =P\allHamming(\Aactiona'),
\end{align}
for some permutation matrix $P\in\{0,1\}^{|\trueSymbols|\times|\trueSymbols|}$.
\end{lemma}

\begin{proof}
The cost function, as written in the form of \eqref{eq:cost2}, depends on $\Aactiona\in\symbolsSet$ only with respect to the distances between $\Aactiona\in\symbolsSet$ and the encoded codewords $\trueSymbol\in\trueSymbols$ while the specific identities of the encoded codewords do not impact the cost function. Correspondingly, any permutation of the distances across the encoded codewords would lead to the same cost. 
\end{proof}

In order to facilitate the analysis of the equilibrium, we, next, seek to write the cost function \eqref{eq:cost3} in a compact form.

\subsection{Equilibrium in Compact Form}
Up to now, we have focused on the objectives except $O3)$, on which \player{A}'s strategy does not have direct impact. Similar to \eqref{eq:piTp}, via the auxiliary functions $\rho(\cdot)$ and $\allHamming(\cdot)$, we can write $O3)$ as
\begin{align}
\factor{3} \sum_{j=0}^{d_o}\Dmixed_j &\sum_{\trueSymbol\in\trueSymbols}\hspace{-.1in}\sum\limits_{\substack{\noisySymbol\in\decodableSymbols\\ \ni \Hamming(\noisySymbol,\trueSymbol)=j}}\hspace{-.1in}p(\noisySymbol | \trueSymbol)p(\trueSymbol) \nn\\
&= \factor{3} \sum_{j=0}^{d_o}\Dmixed_j \sum_{x_o\in\trueSymbols}p(\trueSymbol)\sum_{\tilde{x}_o\in\trueSymbols} \rho(\Hamming(\tilde{x}_o, x_o),j)\\
&= \factor{3} \sum_{j=0}^{d_o}\Dmixed_j \sum_{x_o\in\trueSymbols}p(\trueSymbol)\sum_{t=1}^{|\trueSymbols|} \rho(\allHamming_t( x_o),j).\label{eq:O3}
\end{align}
By including \eqref{eq:O3} in \eqref{eq:cost2} and invoking Lemma \ref{lem:equivalence}, we can write the cost function in a way that confines the impact of the channel into the auxiliary metric $\rho(\cdot)$:
\begin{align}
U(\Amixed,\Dmixed) &= \sum_{\Aaction\in\AactionSpace^d/\sim}\sum_{j=0}^{d_o} \Amixed_{\Aaction}\Dmixed_j\bigg(-\factor{1}\sum_{t=1}^{|\trueSymbols|}\rho(\allHamming_t(\Aactiona),j)\bigg)\nn\\
&+\sum_{\Aaction\in\AactionSpace^d/\sim} \Amixed_{\Aaction}\bigg(\reward(\allHamming(\Aactiona)) + \factor{1}\sum_{j=0}^{d_o}\sum_{t=1}^{|\trueSymbols|}\rho(\allHamming_t(\Aactiona),j)\bigg)\nn\\
&+\sum_{j=0}^{d_o}\Dmixed_j \bigg(\factor{3}\sum_{x_o\in\trueSymbols}p(\trueSymbol)\sum_{t=1}^{|\trueSymbols|} \rho(\allHamming_t( x_o),j)\bigg).\label{eq:cost3}
\end{align}

Consider a certain order over the quotient space $\AactionSpace^d/\sim$ such that, with a slight abuse of notation, $\Amixed_i$ corresponds to the mixed strategy for the $i$th action in $\AactionSpace^d/\sim$. For notational simplicity, we also let $\kappa := |\,\AactionSpace^d/\sim|$ and $\tau := |\trueSymbols|$. Then, \eqref{eq:cost3} can be written as
\begin{align}
U(\Amixed,\Dmixed) &= \sum_{i=1}^{\kappa}\sum_{j=0}^{d_o} \Amixed_i \Dmixed_j\bigg(-\factor{1}\sum_{t=1}^{\tau}\rho(\allHamming_t(\Aactiona^i),j)\bigg)\nn\\
&+\sum_{i=1}^{\kappa} \Amixed_{i}\bigg(\reward(\allHamming(\Aactiona^i)) + \factor{1}\sum_{j=0}^{d_o}\sum_{t=1}^{\tau}\rho(\allHamming_t(\Aactiona^i),j)\bigg)\nn\\
&+\sum_{j=0}^{d_o}\Dmixed_j \bigg(\factor{3}\sum_{x_o\in\trueSymbols}p(\trueSymbol)\sum_{t=1}^{\tau} \rho(\allHamming_t( x_o),j)\bigg).\label{eq:cost4}
\end{align}
which can also be transformed into a compact vectoral form. To this end, we define the vectors $\reward \in \R^{\kappa}$ and $p_o\in\R^{\tau}$, whose $i$th entries are given by $\reward(\allHamming(\Aactiona^i))$ and $p(\trueSymbol^i)$, respectively. We also introduce the matrices $\Phi\in\R^{\kappa\times(d_o+1)}$ and $\Phi_o\in\R^{\tau\times(d_o+1)}$ whose $i$th row and $(j+1)$th column entries are given by
\begin{align}\label{eq:entries}
\sum_{t=1}^{\tau} \rho(\allHamming_t(\Aactiona^i),j)\mbox{ and }
\sum_{t=1}^{\tau} \rho(\allHamming_t(\trueSymbol^i),j), 
\end{align}
respectively. We note the shift at the column entries since we have $\Dmixed_j$, $j=0,\ldots,d_o$ instead of $1,\ldots,d_o+1$. Then, we can write \eqref{eq:cost3} as
\begin{align}\label{eq:compact}
U(\Amixed,\Dmixed) = -\factor{1}\Amixed'\Phi\Dmixed + \Amixed'(r + \factor{1}\Phi\vOnes) + \factor{3}p_o'\Phi_o\Dmixed,
\end{align}
which facilitates the computation of the equilibrium. However, the size of $\AactionSpace^d/\sim$ can lead to computational issues for long codewords, i.e., large $n$, even though it has relatively smaller size compared to $\AactionSpace$ without losing any generality as shown in Lemma \ref{lem:equivalence}. To mitigate this issue, in the following, we relax the attack space so that the size of the problem can be reduced further based on the derived equivalence relation \eqref{eq:equivalence}.

\subsection{Relaxing Attack Space at Large Scales}\label{sec:relax}

The cost function in the compact form \eqref{eq:compact} implies that we need to focus on the first and second additive terms that include \player{A}'s mixed strategy in order to reduce \player{A}'s strategy space. Note that on those additive terms, $\Amixed$ is multiplied by the matrix $\Phi\in\R^{\kappa\times(d_o+1)}$ and the vector $r\in\R^{\kappa}$. We can seek to exploit certain properties of $\Phi\in\R^{\kappa\times(d_o+1)}$ and $r\in\R^{\kappa}$. To this end, we will first show that the matrix $\Phi\in\R^{\kappa\times(d_o+1)}$ can be written as in \eqref{eq:Phi}, where $\delta_i\in\Z^{n+1}$ is a vector which can be viewed as the histogram of the distances from $\Aactiona^i$ to the encoded codewords, i.e., $\allHamming(\Aactiona^i)$. Next, we will examine $\allHamming(\Aactiona^i)$ in order to formulate necessary conditions on the histogram $\delta_i$. By only considering those necessary conditions, we relax \player{A}'s strategy space such that he/she selects a mixed strategy from a strategy space with substantially smaller size. Now, we provide the technical details step by step.

\noindent
{\bf Step-$1$. A Closer Look at the Matrix $\Phi$:}
Recall that the $i$th row and the $(j+1)$th column entry of $\Phi\in\R^{\kappa\times(d_o+1)}$ is given by
\begin{equation}
\sum_{t=1}^{\tau} \rho(\allHamming_t(\Aactiona^i),j),
\end{equation}
where the summation is taken across all the encoded codewords. However, we can separate this summation into sub-summations with respect to the distance from $\Aactiona^i$ to the encoded codewords. In particular, we have
\begin{align}
\sum_{t=1}^{\tau} \rho(\allHamming_t(\Aactiona^i),j) &= \sum_{m=0}^n \sum\limits_{\substack{t\in\{1,\ldots,\tau\}\\\ni \Hamming(\Aactiona^i,\trueSymbol^t) = m}} \rho(\allHamming_t(\Aactiona^i),j)\\
&= \sum_{m=0}^n\delta_i^m \rho(m,j),\label{eq:subsum}
\end{align}
where $\delta_i^m\in\Z$ denotes the number of encoded codewords that are $m$ symbols away from $\Aactiona^i$, and the second line follows since $\allHamming_t(\Aactiona^i)=m$ for all $t$ that satisfies $\Hamming(\Aactiona^i,\trueSymbol^t) = m$. Correspondingly, for each $\Aactiona^i\in\symbolsSet$, $i=1,\ldots,\kappa$, we define the following $n+1$ dimensional vector
\begin{equation}\label{eq:deltaVector}
\delta_i := \begin{bmatrix}\delta_i^0 & \cdots & \delta_i^n \end{bmatrix}'.
\end{equation}
Then, \eqref{eq:subsum} and \eqref{eq:deltaVector} yield that $\Phi\in\R^{\kappa\times(d_o+1)}$ can be written as
\begin{align}\label{eq:Phi}
\Phi = \begin{bmatrix} \delta_1' \\ \vdots \\ \delta_{\kappa}' \end{bmatrix}\underbrace{\begin{bmatrix} \rho(0,0) & \cdots & \rho(0,d_o) \\ 
\vdots & \ddots & \vdots \\
\rho(n,0) & \ldots & \rho(n,d_o) \end{bmatrix}}_{=: R}.
\end{align}

Note that all the entries of $\delta_i\in\Z^{n+1}$, $i=1,\ldots,\kappa$, are non-negative integers and sum to the number of all encoded codewords $\tau$. However, these are not necessarily sufficient conditions. Note also that we can view the vector $\delta_i$ as the histogram of the entries of $\allHamming(\Aactiona^i)$. Based on this observation, in the following, we seek for tighter necessary conditions on $\delta_i$ by examining $\allHamming(\Aactiona^i)$.

\noindent
{\bf Step-$2$. An Upper Bound on $\min\allHamming(\cdot)$:}
We first examine the minimum possible distance between an arbitrary codeword and an encoded codeword. Note that a codeword consists of the message and redundantly added symbols:
\begin{equation}
\overbrace{\bigg[\underbrace{\big[ \mbox{message}\big]}_{\in\Sigma^k} \big[\mbox{redundant~symbols}\big]\bigg]}^{\in\symbolsSet}.
\end{equation}
For each message in $\Sigma^k$, there exists a unique encoded codeword. Correspondingly the minimum distance between an arbitrary codeword $\Aactiona\in\symbolsSet$ and encoded codewords, i.e., $\min\allHamming(\Aactiona)$, can be at most $n-k$ since the message part of $\Aactiona$ matches with at least one encoded codeword completely. Therefore, formally, we have
\begin{equation}\label{eq:allbound}
0\leq \min\allHamming(\Aactiona) \leq n-k\;\forall \Aactiona\in\symbolsSet.
\end{equation}

\noindent
{\bf Step-$3$. A Gap in the Ordered $\{\allHamming(\cdot)\}$:}
Since the codewords are encoded such that they are distributed across $\symbolsSet$ with maximum distance in between them, if an arbitrary codeword is relatively close to one of the encoded codewords, e.g., if it is inside the decodable region, the distances between that arbitrary codeword and the other encoded codewords are relatively large. In other words, when we list all the distances from that arbitrary codeword to the encoded codewords in ascending order, then there will be a jump between the distance to the closest one and the distance to the second closest one. For example, if $\Aactiona=\trueSymbol$, then there is no other encoded codeword within a diameter of $d-1$ symbols away from $\Aactiona$.

Particularly, if $\Aactiona\in\symbolsSet$ is in a decodable region of an encoded codeword, e.g., $\trueSymbol\in\trueSymbols$; then $\trueSymbol$ is the encoded codeword closest to $\Aactiona$, i.e., $\min\allHamming(\Aactiona) = \Hamming(\trueSymbol,\Aactiona)$, and there exists only that encoded codeword within $d-\min\allHamming(\Aactiona)-1$ diameter.

\noindent
{\bf Step-$4$. A Contiguousness Assumption on the Ordered $\{\allHamming(\cdot)\}$:}
We have formulated certain necessary conditions on the distance from an arbitrary codeword to the closest and second closest encoded codewords. For the distances to the other encoded codewords, we observe that at large scales, the number of messages $q^k$, i.e., the number of encoded codewords, is significantly larger than the length of the codewords $n$. We suppose that if $\min\allHamming(\Aactiona)\leq d_o$, then there exists at least one encoded codeword at the distances $d-\min\allHamming(\Aactiona),\ldots,n$. Otherwise, i.e., if $\min\allHamming(\Aactiona)>d_o$, there exists at least one encoded codeword at all the distances starting from the closest one $\min \allHamming(\Aactiona)$ to $n$. In particular, formally, we suppose that
\begin{equation}\label{eq:assumption}
\big\{\max\{d-\min\allHamming(\Aactiona),\min\allHamming(\Aactiona)\},\ldots,n\big\} \subset \allHamming(\Aactiona)
\end{equation} 
since $d-\min\allHamming(\Aactiona) \geq \min\allHamming(\Aactiona)$ if $\Aactiona$ is in the decodable region of an encoded codeword, i.e., $\min\allHamming(\Aactiona)\leq d_o = \lfloor (d-1)/2 \rfloor$. 

\noindent
{\bf Step-$5$. The Histogram Under Necessary Conditions:}
Based on the necessary conditions derived in Steps $2$-$4$, in the following, we formulate the necessary conditions on $\delta_i$ under two cases depending on $\min \allHamming(\Aactiona^i)$. If $\Aactiona^i\in\symbolsSet$ is in a decodable region, i.e., $\min \allHamming(\Aactiona^i) \leq d_o$, then we have
\begin{equation}
\delta_i^m = \left\{\begin{array}{ll} 
1 & \mbox{if } m=\min\allHamming(\Aactiona^i)\\
0 & \mbox{if } m\in\{0,\ldots,d-\min\allHamming(\Aactiona^i)-1\}, m\neq \min\allHamming(\Aactiona^i)\\
* & \mbox{otherwise}
\end{array}\right.
\end{equation}
where $*$ corresponds to an unspecified positive integer. If $\Aactiona^i\in\symbolsSet$ is not in any decodable region, i.e., $\min \allHamming(\Aactiona^i) > d_o$, then we have
\begin{equation}
\delta_i^m = \left\{\begin{array}{ll}
0 & \mbox{if } m\in\{0,\ldots,\min\allHamming(\Aactiona^i)-1\}\\
* & \mbox{otherwise}
\end{array}\right.
\end{equation}

\noindent
{\bf Step-$6$. Approximation Under Necessary Conditions:}
Note that the unspecified entries of $\delta_i$ may not necessarily take arbitrary values; however, we will relax this and suppose that the unspecified entries can be set to arbitrary values by \player{A} as long as they are all positive and all entries sum to $\tau$. As an illustration, when we concatenate $\delta_i$ for different scenarios where $\min\allHamming(\Aactiona^i)$ varies from $0$ to $n-k$, we obtain the following $(n+1)\times(n-k+1)$ matrix:
\begin{equation}\label{eq:illustration}
\begin{array}{r} 0 \\ 1 \\ \vdots \\ d_o \\ d_o+1 \\ \vdots \\ d-1 \\ d \\ \vdots \\n \end{array}\left[
\begin{array}{ccccccc}
1&0&\cdots&0&0&\cdots&0\\
0&1&\cdots&0&0&\cdots&0\\
\vdots&\vdots&\ddots&\vdots&\vdots& & \vdots\\
0&0&\cdots&1&0&\cdots&0\\
0&0&\cdots&0/*&*&\cdots&0\\
\vdots&\vdots&\udots&\vdots&\vdots&\ddots&\vdots\\
0&*&\cdots&*&*&\cdots&*\\
*&*&\cdots&*&*&\cdots&*\\
\vdots&\vdots& &\vdots&\vdots& & \vdots\\
*&*&\cdots&*&*&\cdots&*
\end{array}\right],
\end{equation}
where the entry denoted by $0/*$ is $0$ if $d$ is even, and is an unspecified positive integer if $d$ is odd. For example, the first column corresponds to $\Aactiona^i$ whose $\min\allHamming(\Aactiona^i)=0$, which yields that the second closest encoded codeword can be as close as $d-\min\allHamming(\Aactiona^i)=d$ symbols away. 

\begin{figure*}[t!]
\normalsize
\setcounter{MYtempeqncnt}{\value{equation}}
\setcounter{equation}{36}
\begin{align}\label{eq:Lambda}
\Lambda:=\left[\begin{array}{ccc:ccc:c:ccc:ccc:c:ccc} 
1 & 1 & \cdots & 0 & 0 & \cdots & \cdots & 0 & 0 & \cdots & 0 & 0 & \cdots&\cdots&0&0&\cdots\\
0 & 0 & \cdots & 1 & 1 & \cdots & \cdots & 0 & 0 & \cdots & 0 & 0 & \cdots&\cdots&0&0&\cdots\\
\vdots & \vdots & & \vdots & \vdots & &\ddots &\vdots & \vdots & & \vdots&\vdots &&&\vdots&\vdots&\\
0 & 0 & \cdots & 0 & 0 & \cdots & \cdots & 1 & 1 & \cdots & 0& 0 & \cdots& \cdots&0&0&\cdots\\
\cdashline{8-13}
0 & 0 & \cdots & 0 & 0 & \cdots & \cdots & \lambda_{d_o} & 1 & \cdots&\lambda_{d_o+1}& 1 & \cdots &\cdots&0&0&\cdots\\
0 & 0 & \cdots & 0 & 0 & \cdots & \udots & 1 & \lambda_{d_o} & \cdots&1&\lambda_{d_o+1}&\cdots&\cdots&0&0&\cdots\\
\vdots & \vdots & &\vdots & \vdots & & &\vdots&\vdots& &\vdots&\vdots&&\ddots&\vdots&\vdots&\\
\cdashline{4-6}\cdashline{15-17}
0 & 0 & \cdots & \lambda_1 & 1 & \cdots& \cdots & 1 & 1 & \cdots & 1&1&\cdots&\cdots&\lambda_{n-k}&1&\cdots\\
\cdashline{1-3}
\lambda_0 & 1 & \cdots & 1 & \lambda_1 & \cdots& \cdots & 1 & 1 & \cdots& 1&1&\cdots &\cdots&1&\lambda_{n-k}&\cdots\\
1 & \lambda_0 & \cdots & 1 & 1 & \cdots& \cdots & 1 & 1 & \cdots & 1&1&\cdots&\cdots&1&1&\cdots\\
\vdots & \vdots & & \vdots & \vdots & & & \vdots & \vdots & & \vdots&\vdots&&&\vdots&\vdots&\\
1 & 1 & \cdots & 1 & 1 & \cdots & \cdots & 1 & 1 & \cdots&1&1&\cdots&\cdots&1&1&\cdots
\end{array}\right]
\end{align}
\begin{align*}
\lambda_i = 
\left\{
\begin{array}{ll}
\tau-(n-d+i+1)&\mbox{if $i\leq d_o$}\\
\tau-(n-i)&\mbox{o.w.}
\end{array}\right.
\end{align*}
\setcounter{equation}{\value{MYtempeqncnt}}
\hrulefill
\vspace*{4pt}
\end{figure*}

\noindent
{\bf Step-$7$. \player{A}'s Relaxed Strategy Space:}
Based on the relaxation that the unspecified entries can take any values, we seek to reduce \player{A}'s strategy space, which is the main reason of all the steps we have taken up to now. To this end, we first recall that the unspecified entries are all positive and add up to a certain number, which is $\tau-1$ if $\min\allHamming(\Aactiona^i)\leq d_o$, and $\tau$ otherwise. Let us consider an arbitrary column in \eqref{eq:illustration}, e.g., $m$th column where $m > d_o$. Then, the set of all such possible $\delta_i$ is given by
\begin{equation}\label{eq:setExt1}
\{\delta\in\Z^{n+1}: \delta^t = 0 \mbox{ if } t<m, \delta^t >0 \mbox{ o.w., and } \vOnes'\delta=\tau\}.
\end{equation}
Correspondingly, the set of extreme points\footnote{We say that a point in a convex set is an extreme point if it cannot be expressed as a convex combination of two other points from that set.} of this set is given by
\begin{equation}\label{eq:setExt2}
\{e_i\in\Z^{n+1}: e_i^i = \tau-(n-m) \mbox{ and } e_i^j = 1 \mbox{ if } j\geq m, j\neq i\}.
\end{equation}
Note that any point in the set \eqref{eq:setExt1} can be expressed as a convex combination of its extreme points identified in \eqref{eq:setExt2}.

Therefore, we can express any convex combination of $\delta_i$, $i=1,\ldots,\kappa$, by a convex combination of the columns of the matrix\footnote{We suppose that $d$ is odd, i.e., $d_o = (d-1)/2$. The matrix for the cases where $d$ is even can be computed accordingly.} $\Lambda\in\Z^{(n+1)\times\nu}$ defined in \eqref{eq:Lambda}, where\addtocounter{equation}{1}
\begin{equation}\label{eq:nu}
\nu :=\sum_{i=0}^{d_o} (n-d+i+1) + \sum_{j=d_o+1}^{n-k}(n-j+1)
\end{equation}
and $n-k\geq d_o+1$ by the Singleton bound \eqref{eq:Singleton}. In other words, under the relaxation, for any given mixed strategy $\Amixed$ across $\AactionSpace^d/\sim$, there exists a mixed strategy $\beta\in\Delta^{\nu-1}$ over the columns of $\Lambda$ such that we have
\begin{equation}
\begin{bmatrix} \delta_1 & \cdots & \delta_{\kappa} \end{bmatrix} \Amixed = \Lambda \beta,
\end{equation}
which, by \eqref{eq:Phi}, yields that
\begin{equation}
\Amixed'\Phi = \beta' \Lambda' R,
\end{equation}
where $R\in\R^{(n+1)\times(d_o+1)}$, as defined in \eqref{eq:Phi}.

\noindent
{\bf Step-$8$. A Closer Look at the Vector $\reward$:}
Next, we seek to compute the reward $\reward(\Aactiona^i)$. Recall that the reward for $\Aactiona^i$ depends only on $\{\allHamming(\Aactiona^i)\}$, as defined in \eqref{eq:Aoopt2}. However, $\{\allHamming(\Aactiona)\}$ depends only on $\min\allHamming(\Aactiona)$ as shown in Steps $2$-$4$. Therefore,   \eqref{eq:assumption} yields that $\reward(\Aactiona^i)$ depends mainly on the distance to the closest encoded codeword. Based on \eqref{eq:Aoopt2} and \eqref{eq:assumption}, we define an auxiliary vector $s\in\R^{n-k+1}$, where $s_i\in\R$, for $i=0,\ldots,n-k$, is given by
\begin{align}
s_i := \max_{\hbar\in\{0,\ldots,n\}}& \factor{2} - \factor{4}\hbar - \factor{2} \sum_{t=0}^{d_o}\rho(\hbar,t)\label{eq:reward_s}\\
\mathrm{s.t.}&\; \hbar = i \vee \hbar \geq \max\{d-i,i+1\},\nn
\end{align}
where $\vee$ denotes the disjunction operation. This yields that $s_i\in\R$ corresponds to the reward when $\min\allHamming(\Aactiona) = i$.

Note that for all $\Aactiona^i$ that have $\min\allHamming(\Aactiona^i) = m$, the associated reward is $s_m$. Therefore, with the mixed strategy $\beta\in\Delta^{\nu-1}$ introduced in Step-$7$, we have
\begin{equation}\label{eq:S}
r'\Amixed = s'\underbrace{\begin{bmatrix} \vOnes' & & \\ & \ddots & \\ & & \vOnes' \end{bmatrix}}_{=:S}\beta,
\end{equation}
where $S\in\R^{(n-k+1)\times \nu}$, and the dimensions of the vector $\vOnes$ at the $i$th row is $n-d+i$ if $i\leq d_o$, and $n-i$ if $i>d_o$.  

\noindent
{\bf Step-$9$. Transforming \player{D}'s Strategy Space to a Simplex at a Higher Dimensional Space:}
Our goal, here, is to transform \player{D}'s strategy into a mixed strategy at a higher dimensional space in order to be able to transform the problem into an LP as will be explained in detail later in this section. To this end, we can view $\Dmixed\in\DmixedSpace$ as \player{D} selects $d_o+1$ mixed strategies over two element sets, e.g., $\{0,1\}$. This yields that \player{D} selects a mixed strategy over the Cartesian product space of these sets, i.e., $\bigtimes_{i=0}^{d_o}\{0,1\}$, which is
\begin{equation}\label{eq:mu}
\mu = 2^{d_o+1}
\end{equation}
dimensional. For example, for $d_o=1$, the corresponding mixed strategy, denoted by $\sigma\in\Delta^{\mu-1}$, is over $\{[1,0,1,0]',[1,0,0,1]',[0,1,1,0]',[0,1,0,1]'\}$. This yields that there exists a matrix $\Pi\in\R^{(d_o+1)\times\mu}$ such that $\Dmixed = \Pi \sigma$.  As an example, for $d_o=1$, we have
\begin{equation}\label{eq:Pi}
\Dmixed = \underbrace{\begin{bmatrix} 1&1&0&0\\
                                          1&0&1&0
                  \end{bmatrix}}_{=\Pi} \sigma.
\end{equation}

\noindent
{\bf Step-$10$. New Compact Form:}
Eventually, for the relaxed attack strategies, we can write \player{D}'s cost function in the following compact form:
\begin{equation}\label{eq:main}
\beta' \Xi \sigma
\end{equation} 
where $\beta\in\Delta^{\nu-1}$, $\sigma\in\Delta^{\mu-1}$, and
\begin{align}\label{eq:Xi}
\Xi:= -\factor{1}\Lambda'R\Pi +S's\vOnes' + \factor{1}\Lambda'R\vOnes\vOnes' + \factor{3}\vOnes p_o'\Phi_o\Pi,
\end{align}
which follows since we have $\vOnes'\beta=1$ and $\vOnes'\sigma$, which yields, e.g., $\beta'S's = \beta'S's \vOnes'\sigma$.

In the following lemma, we provide an LP to compute the best detection rule.

\begin{lemma}
After the relaxation of \player{A}'s strategy space, the best detection rule $\Dmixed_*\in\DmixedSpace$ is given by 
\begin{equation}
\Dmixed_* = \Pi\sigma_*\mbox{ and } \sigma_* =  \frac{\omega_*}{\vOnes'\omega_*},
\end{equation}
where $\omega_*\in\R^{\mu}$ is the solution of the following LP:
\begin{align}\label{eq:LP}
\max_{\omega \in \R^{\mu}} \vOnes'\omega \mbox{ subject to } \gameMatrix_+\omega \leq \vOnes, \, \omega \geq \vZeros,
\end{align} 
where the positive matrix\footnote{We say that a matrix is positive if its all entries are positive.} $\gameMatrix_+ \in \R_+^{\nu\times\mu}$ is defined by 
\begin{equation}
\gameMatrix_+ := \left\{\begin{array}{ll} \gameMatrix & \mbox{if $\gameMatrix$ is a positive matrix} \\ \gameMatrix + (\epsilon - \gameConstant_o)\vOnes\vOnes' & \mbox{otherwise},\end{array}\right.
\end{equation}
where $\epsilon>0$ and $\gameConstant_o\in\R$ is the minimum entry of $\gameMatrix$.
\end{lemma}

\begin{proof}
Note that by definition, we have
\begin{equation}
\min_{\sigma} \max_{\beta} \beta'\gameMatrix \sigma \geq \max_{\beta} \min_{\sigma} \beta' \gameMatrix \sigma.
\end{equation}
We are interested in only the upper value since we seek to compute the Stackelberg equilibrium where \player{A} selects the strategy $\sigma$ by knowing \player{D}'s strategy $\beta$. However, since the objective functions are linear in the optimization arguments while the constraint sets are convex, decoupled, and compact, the minimax theorem \cite{ref:Basar99} shows that
\begin{equation}
\min_{\sigma} \max_{\beta} \beta' \gameMatrix \sigma = 
\max_{\beta} \min_{\sigma} \beta' \gameMatrix \sigma,
\end{equation}
which implies that the upper and lower values of the game are equal and that we have a saddle-point equilibrium in \eqref{eq:main}. Therefore, we can apply rather routine transformation of mixed-strategy equilibrium of zero-sum matrix games into an LP \cite{ref:Basar99} in order to compute the best detection rule. 

A sketch of the routine transformation of \eqref{eq:main} into an LP \cite{ref:Basar99} is as follows: $i)$ we show that the game \eqref{eq:main} is strategically equivalent to a game where the game matrix is a positive matrix; $ii)$ we can write \eqref{eq:main} as the minimization of \player{A}'s best response; $iii)$ we can obtain a certain necessary condition on $\sigma$ in terms of \player{A}'s best response since $\Delta^{\nu-1}$ is a simplex; $iv)$ through a change of variable, we can obtain the equivalent LP \eqref{eq:LP}.
\end{proof}

\begin{corollary}
The solution for the dual problem of \eqref{eq:LP}, i.e.,
\begin{align}\label{eq:LP2}
\min_{\vartheta \in \R^{\nu}} \vOnes'\vartheta \mbox{ subject to } \gameMatrix_+'\vartheta \leq \vOnes, \, \omega \geq \vZeros,
\end{align} 
yields that
\begin{equation}
\beta_* = \frac{\vartheta_*}{\vOnes'\vartheta_*}.
\end{equation}
\end{corollary}

In the following section, we analyze the performance of the proposed detection mechanism numerically for various scenarios.

\section{Numerical Examples} \label{sec:examples}

The proposed framework can be applied to any linear block code since the analytical results are based only on the abstraction of the code, i.e., $[n,k,d]_q$. Therefore, in order to compute the detection rule, specific to the underlying encoding-decoding scheme, we need the configuration of the code, i.e., $[n,k,d]_q$, and the matrix $R\in\R^{(n+1)\times(d_o+1)}$, as defined in \eqref{eq:Phi}. At run-time, the detection mechanism triggers an alert based only on the number of mismatched symbols. 

As numerical examples, in this section, we examine the performance of the proposed detection mechanisms for the smart codes that are constructed via Reed-Solomon coding \cite{ref:Reed60}. Particularly, Reed-Solomon code (RS-Code) is a maximum distance separable code that maximizes the minimum distance between any two distinct codewords within the general class of linear block codes, and it has widely used applications, e.g., QR codes. The minimum distance of RS-Code $[n,k,d]_q$ is given by
\begin{equation}
d = n-k+1,
\end{equation} 
which is the Singleton bound \eqref{eq:Singleton} for the linear block codes. In practical implementations, the alphabet size is in general selected a prime power and length of codeword is set $n<q$, e.g., often $n=q-1$.

\begin{table}[t!]
\caption{Configuration of the RS-Codes}
\renewcommand{\arraystretch}{1.5}
\begin{center}
\begin{tabular}{r|c|c|c|c|c}
RS-Code				& $\#$ distinct road signs & $\#$ bits & $\nu$ & $\mu$ & $d_o$\\
\hline
$[7,5,3]_8$ 		& $8^5=32768$ 	& $21$ & $17$ &$4$	&$1$\\
$[7,3,5]_8$ 		& $8^3=512$ 		& $21$ & $21$ &$8$	&$2$\\
$[11,5,7]_{16}$ 		& $16^5=1,048,576$	& $44$ & $47$ &$16$	&$3$\\
$[11,3,9]_{16}$ 		& $16^3=4096$		& $44$ & $47$ &$32$	&$4$\\
$[15,5,11]_{16}$ 	& $16^5=1,048,576$	& $60$ & $85$ &$64$	&$5$\\ 
$[15,3,13]_{16}$ 	& $16^3=4096$		& $60$ & $81$ &$128$	&$6$ 
\end{tabular}
\end{center}
\label{tab:codes}
\end{table}

As illustrative examples, we examine the performance of the proposed detection mechanism for the RS-Codes: $[7,5,3]_8$, $[7,3,5]_8$, $[11,5,7]_{16}$, $[11,3,9]_{16}$, $[15,5,11]_{16}$, and $[15,3,11]_{16}$ such that the corresponding distances for the decodable regions are given by $d_o = 1,2,3,4,5,6$, respectively. For each RS-Code, Table \ref{tab:codes} tabulates the maximum number of distinct road signs that can be encoded, the number of bits in the codeword, i.e., $n\times\log_2 q$ (which can give an idea about the size of the associated smart code), dimensions of the mixed strategies $\beta\in\Delta^{\nu-1}$ and $\sigma\in\Delta^{\mu-1}$, and the decodable distance $d_o\in\Z$. The number of distinct road signs that a code can express is not directly related to the decodable distance. For example, $[11,5,7]_{16}$ can encode as much as around $1$ million distinct road signs, but its decodable distance is $3$, which is less than the decodable distance of $[15,3,13]_{16}$, which can encode as much as $4096$ distinct road signs.

\begin{table}[t!]
\caption{Probability of decoding error/failure for the RS-Codes over Different Channels. Highlighted cells correspond to the error/failure probabilities less than $0.02$.}
\renewcommand{\arraystretch}{1.5}
\begin{center}
\begin{tabular}{r|c|c|c|c}
RS-Code & $\errorProb=0.01$ & $\errorProb=0.05$ & $\errorProb=0.1$ & $\errorProb=0.2$\\
\hline
$[7,5,3]_8$ 		& \cc $0.0020$ & $0.0444$ & $0.1497$ & $0.4233$\\
$[7,3,5]_8$ 		& \cc $0.0000$ & \cc $0.0038$ & $0.0257$ & $0.1480$\\
$[11,5,7]_{16}$ 		& \cc $0.0000$ & \cc $0.0016$ & \cc $0.0185$ & $0.1611$\\
$[11,3,9]_{16}$ 		& \cc $0.0000$ & \cc $0.0001$ & \cc $0.0028$ & $0.0504$\\
$[15,5,11]_{16}$ 	& \cc $0.0000$ & \cc $0.0001$ & \cc $0.0022$ & $0.0611$\\ 
$[15,3,13]_{16}$ 	& \cc $0.0000$ & \cc $0.0000$ & \cc $0.0003$ & \cc $0.0181$ 
\end{tabular}
\end{center}
\label{tab:errorRate}
\end{table}

In order to examine the performance across a range of relative small to high noise channels, we consider $4$ different channels with probabilities of symbol errors: $\errorProb = 0.01,0.05,0.1,0.2$. For each channel, in Table \ref{tab:errorRate}, we tabulate the probability of decoding error/failure for the RS-Codes. We highlight the error/failure probabilities that are less than $0.02$. Note that a high error rate yields that the associated code is not reliable even when there is no adversarial intervention. Correspondingly, if a code leads to a higher error/failure probability, then we can prefer codes that include more redundancy to improve reliability. 

Next, we compare the reliability of the smart road signs with respect to the cost metric \eqref{eq:cost} for the cases with and without the proposed detection mechanism.
For example, we set all the road signs equally likely and we use the Monte Carlo method to compute $R\in\R^{(n+1)\times(d_o+1)}$ over $10^6$ independent trials. We set the multiplicative factors as
\begin{equation}
\factor{1}=\factor{2}=\factor{4} = 100 \mbox{ and }\factor{3} = \tau,
\end{equation}
where we set the weight of the objective $O3)$, i.e., false alarm cost, in order to keep the probability of false alarm at a certain range, e.g., less than $10\%$, for the scenarios where the code has error/failure probability less than $0.02$. In order to solve the LPs numerically, we use CVX, a package for specifying and solving convex programs \cite{ref:Grant08,ref:CVX}.

\begin{table}[t!]
\caption{The conservative cost if there were no detection mechanism.}
\renewcommand{\arraystretch}{1.5}
\begin{center}
\begin{tabular}{r|r|r|r|r}
RS-Code & $\errorProb=0.01$ & $\errorProb=0.05$ & $\errorProb=0.1$ & $\errorProb=0.2$\\
\hline
$[7,5,3]_8$ 		& \cc $9,000$ & $36,251$ & $55,564$ & $61,491$\\
$[7,3,5]_8$ 		& \cc $112$ & \cc $778$ & $1,306$ & $1,683$\\
$[11,5,7]_{16}$ 		& \cc $257,750$ & \cc $980,210$ & \cc $1,331,300$ & $1,179,500$\\
$[11,3,9]_{16}$ 		& \cc $913$ & \cc $4,750$ & \cc $6,836$ & $6,901$\\
$[15,5,11]_{16}$ 	& \cc $379,300$ & \cc $1,313,600$ & \cc $1,628,300$ & $1,152,200$\\ 
$[15,3,13]_{16}$ 	& \cc $1,280$ & \cc $5,774$ & \cc $7,718$ & \cc $6,108$ 
\end{tabular}
\end{center}
\label{tab:no_detector}
\end{table}

\begin{table}[t!]
\caption{The conservative cost if there were the proposed detection mechanism.}
\renewcommand{\arraystretch}{1.5}
\begin{center}
\begin{tabular}{r|c|c|c|c}
RS-Code & $\errorProb=0.01$ & $\errorProb=0.05$ & $\errorProb=0.1$ & $\errorProb=0.2$\\
\hline
$[7,5,3]_8$ 		& \cc $2,246$ & $9,579$ & $16,167$ & $22,847$\\
$[7,3,5]_8$ 		& \cc $100$ & \cc $112$ & $145$ & $214$\\
$[11,5,7]_{16}$ 		& \cc $279$ & \cc $19,394$ & \cc $88,599$ & $261,270$\\
$[11,3,9]_{16}$ 		& \cc $100$ & \cc $105$ & \cc $160$ & $530$\\
$[15,5,11]_{16}$ 	& \cc $100$ & \cc $5,549$ & \cc $39,194$ & $145,690$\\ 
$[15,3,13]_{16}$ 	& \cc $100$ & \cc $100$ & \cc $107$ & \cc $262$ 
\end{tabular}
\end{center}
\label{tab:with_detector}
\end{table}

\begin{table}[t!]
\caption{Probability of False Alarms}
\renewcommand{\arraystretch}{1.5}
\begin{center}
\begin{tabular}{r|c|c|c|c}
RS-Code & $\errorProb=0.01$ & $\errorProb=0.05$ & $\errorProb=0.1$ & $\errorProb=0.2$\\
\hline
$[7,5,3]_8$ 		& \cc $0.0657$ & $0.2884$ & $0.4822$ & $0.6724$\\
$[7,3,5]_8$ 		& \cc $0.0001$ & \cc $0.0296$ & $0.1082$ & $0.2711$\\
$[11,5,7]_{16}$ 		& \cc $0.0002$ & \cc $0.0135$ & \cc $0.0708$ & $0.2259$\\
$[11,3,9]_{16}$ 		& \cc $0.0000$ & \cc $0.0013$ & \cc $0.0150$ & $0.1077$\\
$[15,5,11]_{16}$ 	& \cc $0.0000$ & \cc $0.0052$ & \cc $0.0106$ & $0.1031$\\ 
$[15,3,13]_{16}$ 	& \cc $0.0000$ & \cc $0.0000$ & \cc $0.0018$ & \cc $0.0406$ 
\end{tabular}
\end{center}
\label{tab:false_alarm}
\end{table}

In order to compute the (conservative) cost for the scenarios where there is no detection mechanism, we compute
\begin{equation}
\max_{\beta\in\Delta^{\nu-1}} \beta' \Xi \sigma_o,
\end{equation}
where
\begin{equation}
\sigma_o := \begin{bmatrix} 0 & \ldots & 0 & 1 \end{bmatrix}'.
\end{equation}
Particularly, since the right most column of $\Pi\in\R^{(do+1)\times \mu}$ is a zero vector, as exemplified in \eqref{eq:Pi}, $\sigma_o\in\Delta^{\mu-1}$ yields that $\Dmixed = \Pi\sigma_o = \pmb{0}$. In Table \ref{tab:no_detector}, we tabulate the (conservative) cost of the codes over the channels examined if there were no detection mechanism. Note that we highlight the cells corresponding to the error/failure probabilities less than $0.02$ in order to distinguish the scenarios where the associated RS-Code can be used reliably. In Table \ref{tab:with_detector}, we tabulate the conservative cost of the codes if there were the proposed detection mechanism. The corresponding false alarm rates are provided in Table \ref{tab:false_alarm}. A comparison of Tables \ref{tab:no_detector} and  \ref{tab:with_detector} shows a substantial decrease in the conservative cost at the expense of a false alarm rate less than $10\%$ in the scenarios where the error/failure probability is less than $0.02$. 

\begin{remark}
We propose a way to relax certain constraints on the attack space to mitigate the scalability issue. It can be possible to obtain tighter approximations by considering tighter necessary conditions on \player{A}'s actions; however, this would also increase computational complexity.
\end{remark}

\begin{figure}[t!]
\centering
\includegraphics[width=.48\textwidth]{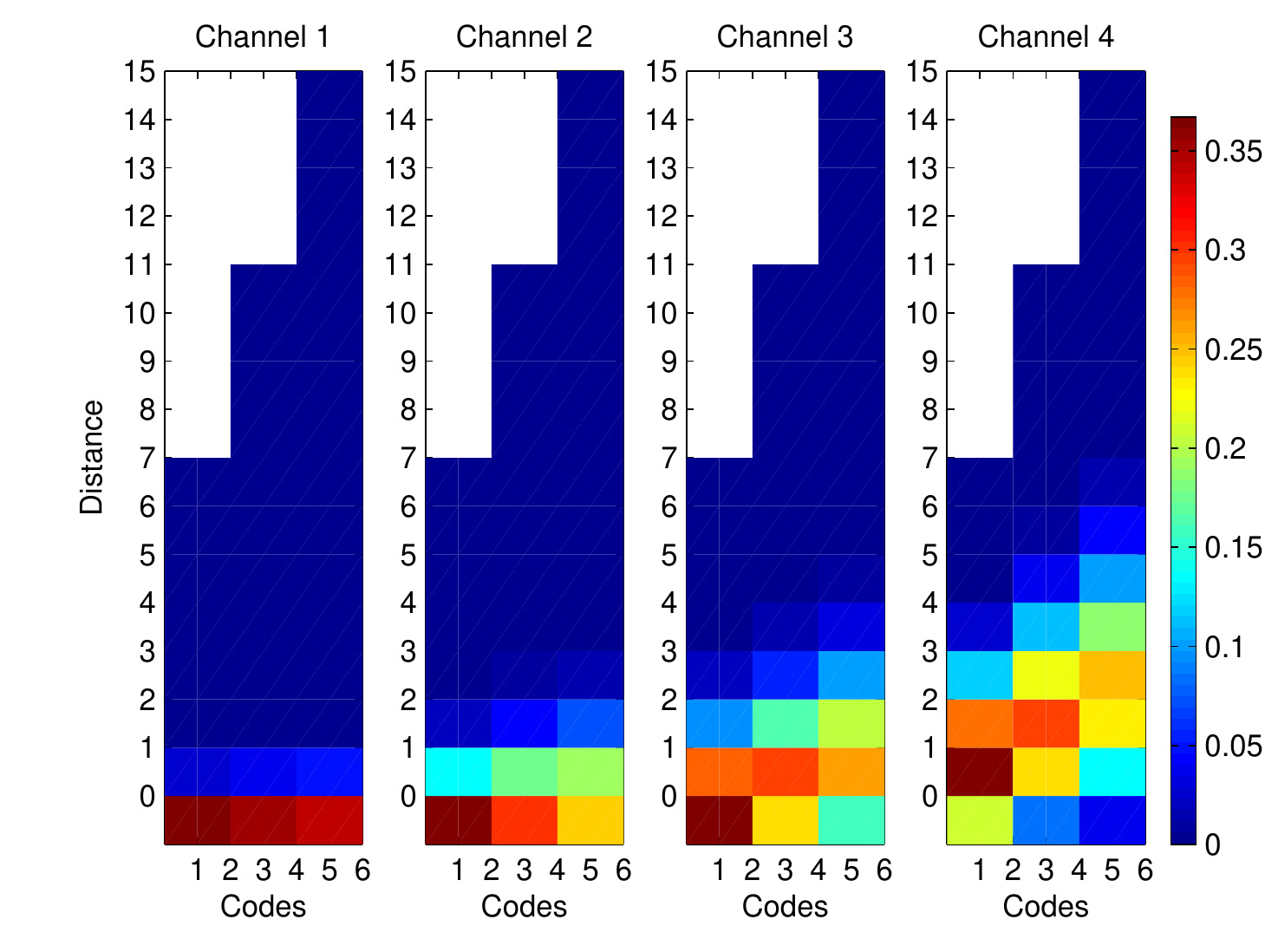}
\caption{The probability of the number of symbol errors for the RS-Codes enumerated in the order of Table \ref{tab:codes} across the channels enumerated with respect to the probability of symbol errors $\errorProb=0.01,0.05,0.1,0.2$ in that order.}
\label{fig:noise}
\end{figure}

\begin{figure}[t!]
\centering
\includegraphics[width=.48\textwidth]{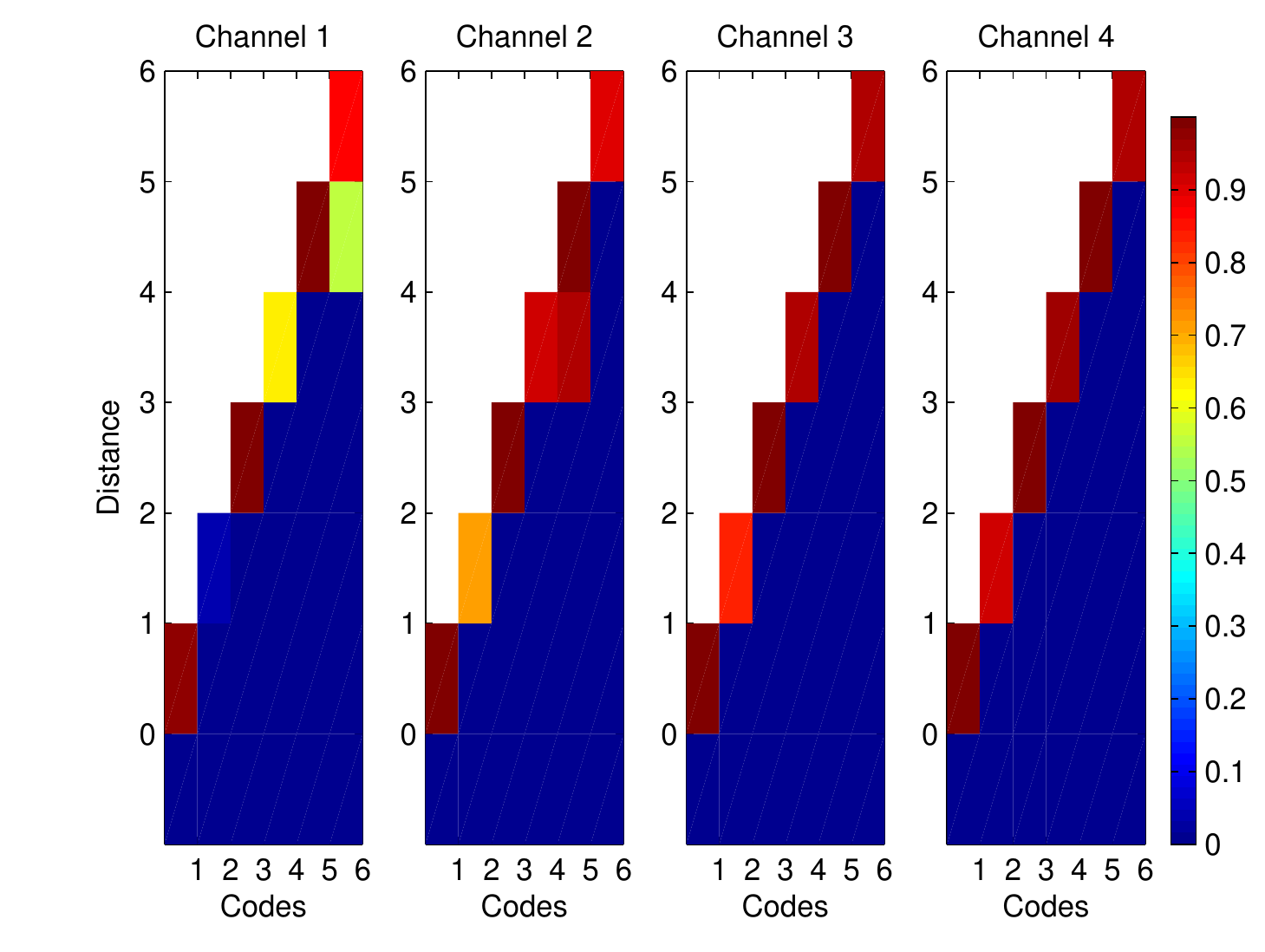}
\caption{The detection rule $\Dmixed_*\in\DmixedSpace$ for the RS-Codes enumerated across the channels enumerated.}
\label{fig:detectionRule}
\end{figure}

\begin{figure}[t!]
\centering
\includegraphics[width=.48\textwidth]{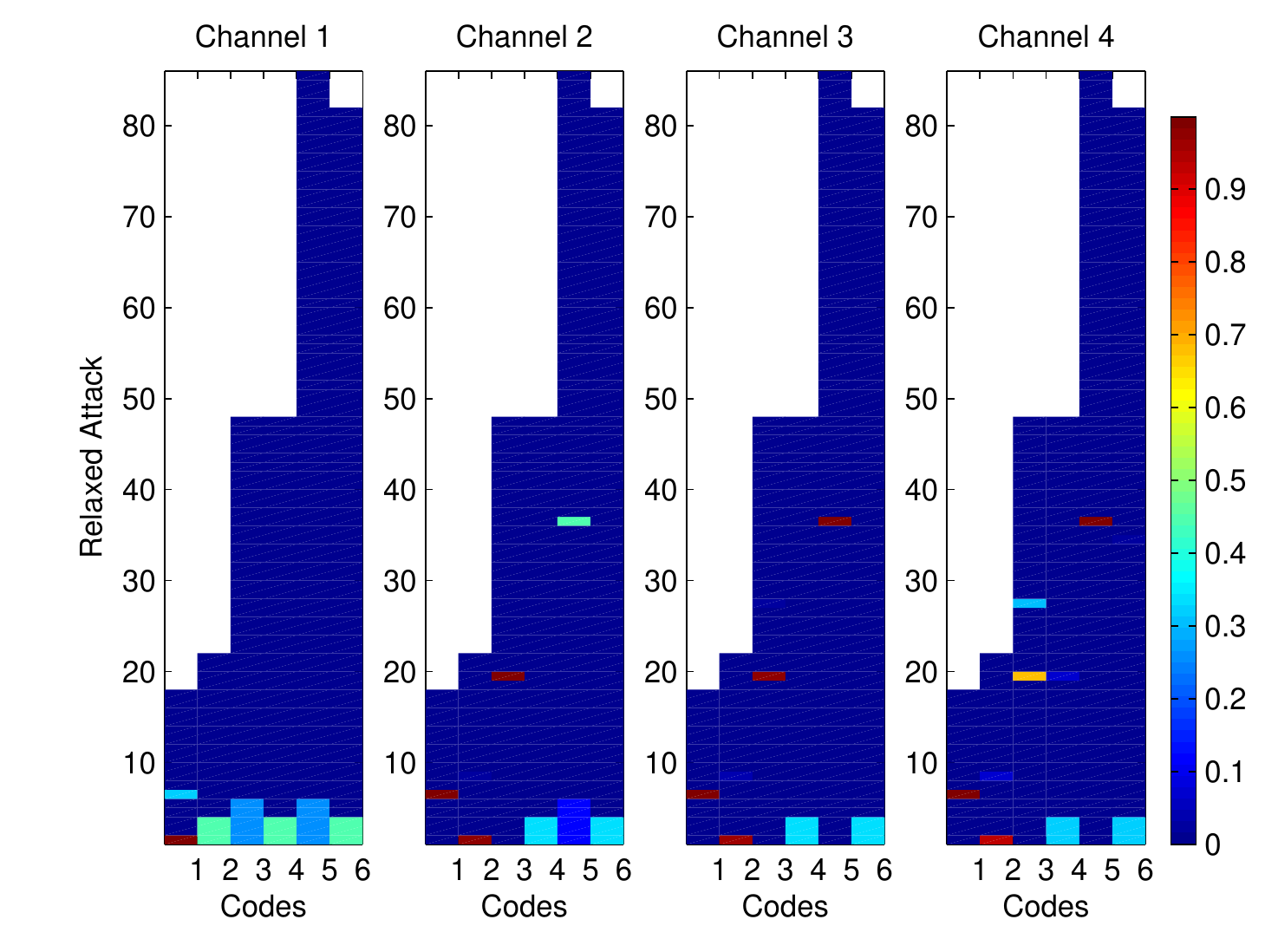}
\caption{The relaxed attacker strategy $\beta_*\in\Delta^{\nu-1}$, which is a mixed strategy over the columns of $\Lambda\in\R^{(n+1)\times\nu}$, as defined in \eqref{eq:Lambda}.}
\label{fig:relaxedAttack}
\end{figure}

Furthermore, in Figs \ref{fig:noise}, \ref{fig:detectionRule}, and \ref{fig:relaxedAttack}, we provide the probability of the number of symbol errors due to channels, the proposed detection rule, and the relax attack strategy, respectively. We observe that \player{D} triggers alerts if the error rate is relatively high in general, which turns out to restrain the (powerful) attacker to put more weight on the left most columns of $\Lambda\in\R^{(n+1)\times\nu}$, as defined in \eqref{eq:Lambda}, in his/her (relaxed) attack strategy. Particularly, at the equilibrium, the powerful attacker ends up crafting the smart code relatively more aggressively similar to the choice $C3)$ as discussed in Subsection \ref{sec:equivalence}. Depending on the channel and the configuration of the code, the optimal detection rules can vary. Through the proposed mechanism, based on a game theoretical analysis, we can compute the best detection rule efficiently and systematically even at scales of around $1$ million distinct road signs using an average personal computer without difficulty.

\section{Conclusion} \label{sec:conclusion}

A future trend in intelligent transportation systems is smart road signs equipped with smart codes. In addition to incorporating relatively larger amount of information, smart codes constructed via error-correction methods can provide robustness against small scale perturbations. We have introduced a game theoretical adversarial intervention detection mechanism for reliable smart road signs against threats that can perturb the smart codes at small or large scales intelligently. While designing the detection mechanism, we have considered multiple performance metrics regarding the cost associated with losing the opportunity of preventing future attacks by not being able to detect the attack, the cost associated with adversary-induced decoding error or failure, the false alarm cost, and the ease of a deceptive perturbation. We have designed the detection rule against the worst-case attacker who maximizes the cost metrics by knowing the designed defense, i.e., under the solution concept of Stackelberg equilibrium where the defender is the leader. We have provided a relaxation on the attacker's strategy space in order to mitigate possible computational issues that might arise while computing the equilibrium when there is a large number of distinct road signs. This has enabled us transform the problem into an LP with considerably small computational complexity. Finally, we have examined the performance numerically over various scenarios. 

The proposed game theoretical framework brings in new research directions for the applications of smart road signs in intelligent transportation systems. In the following, we identify some of these future research directions:
\begin{itemize}
\item We emphasize that sensor fusion where we collect information through several separate sources can lead to more resilient and robust systems \cite{ref:Klein04}. In the future, smart road signs combined with state-of-the-art vision-based road-sign recognition algorithms can provide both reliable and effective recognition by smart vehicles. 
\item A network of smart vehicles can lead to more reliable traffic networks. Particularly, a detection mechanism faces a trade-off between detecting an adversarial intervention and avoiding false alarms. Since a road sign would be encountered by multiple smart vehicles, those vehicles can share the false alarm cost against an attack on the road sign. Similar to herd immunity \cite{ref:Fine11}, a herd of smart vehicles can achieve more reliable road sign recognition. 
\item Additionally, this approach can also be a good fit for other classification problems that can be viewed as a signaling problem, where we can incorporate visual smart codes while transmitting information. For example, computer vision for (warehouse) inventory management \cite{ref:Katircioglu15} or intelligent robotic sorting \cite{ref:Guerin18} would constitute other interesting applications for the framework developed here. 
\end{itemize}

\bibliographystyle{IEEEtran}
\bibliography{ref}

\begin{IEEEbiographynophoto}{Muhammed O. Sayin} received the B.S. and M.S. degrees in electrical and electronics engineering from Bilkent University, Ankara, Turkey, in 2013 and 2015, respectively. He is currently pursuing the Ph.D. degree in electrical and computer engineering from the University of Illinois at Urbana-Champaign (UIUC). His current research interests include dynamic games and decision theory, security, stochastic control, and cyber-physical systems.
\end{IEEEbiographynophoto}

\begin{IEEEbiographynophoto}{Chung-Wei Lin} received the B.S. degree in computer science and information engineering and the M.S. degree in electronics engineering from the National Taiwan University, Taipei, Taiwan. He received the Ph.D. degree in electrical engineering and computer sciences from the University of California, Berkeley, Berkeley, CA, USA. He is an Assistant Professor at the Department of Computer Science and Information Engineering, National National Taiwan University, Taipei, Taiwan. His research includes design, analysis, security, and certification of automotive systems.
\end{IEEEbiographynophoto}

\begin{IEEEbiographynophoto}{Eunsuk Kang} received a Ph.D. degree in computer science from MIT, and a B.S.E. degree from the University of Waterloo in Canada. He is an Assistant Professor in the Institute for Software Research, School of Computer Science at Carnegie Mellon University. His research interests are in software engineering, formal methods, security, and system safety.
\end{IEEEbiographynophoto}

\begin{IEEEbiographynophoto}{Shinichi Shiraishi} (M'00) received the B.S., M.S., and Ph.D. degrees in electronics engineering from Hokkaido University, Sapporo, Japan, in 1997, 1999, and 2002, respectively. He is currently a Group Leader with Toyota InfoTechnology Center, Co., Ltd., Minato-ku, Tokyo, Japan. His research interests include software assurance, software architecture, modeling languages, and design analysis. 
\end{IEEEbiographynophoto}

\begin{IEEEbiographynophoto}{Tamer Ba\c{s}ar} (S'71-M'73-SM'79-F'83-LF'13) is with the University of Illinois at Urbana-Champaign, where he holds the academic positions of  Swanlund Endowed Chair; Center for Advanced Study Professor of  Electrical and Computer Engineering; Research Professor at the Coordinated Science Laboratory; and Research Professor  at the Information Trust Institute. He is also the Director of the Center for Advanced Study.

He received B.S.E.E. from Robert College, Istanbul, and M.S., M.Phil, and Ph.D. from Yale University. He is a member of the US National Academy
of Engineering, member of the  European Academy of Sciences, and Fellow of IEEE, IFAC (International Federation of Automatic Control) and SIAM (Society for Industrial and Applied Mathematics), and has served as president of IEEE CSS (Control Systems  Society), ISDG (International Society of Dynamic Games), and AACC (American Automatic Control Council). He has received several awards and recognitions over the years, including the highest awards of IEEE CSS, IFAC, AACC, and ISDG, the IEEE Control Systems Award, and a number of international honorary doctorates and professorships. He has over 900 publications in systems, control, communications, and dynamic games, including books on non-cooperative dynamic game theory, robust control, network security, wireless and communication networks, and stochastic networked control. He was the Editor-in-Chief of Automatica between 2004 and 2014, and is currently  editor of several book series. His current research interests include stochastic teams, games, and networks; distributed algorithms; security; and cyber-physical systems.
\end{IEEEbiographynophoto}

\vfill

\end{document}